\documentclass[letterpaper]{article} 
\usepackage[preprint]{aaai2027}  
\usepackage[hyphens]{url}  
\usepackage{graphicx} 
\usepackage{natbib}  
\usepackage{caption} 
\usepackage{algorithm}
\usepackage[noend]{algorithmic}

\usepackage{newfloat}
\usepackage{listings}

\usepackage{todonotes}
\usepackage{amssymb}
\usepackage{amsmath}
\usepackage{amsthm}
\usepackage{booktabs}
\usepackage{xcolor}

\newtheorem{property}{Property}
\newtheorem{notation}{Notation}
\newtheorem{proposition}{Proposition}
\newtheorem{definition}{Definition}

\newcommand{\A}{\mathcal{A}}
\newcommand{\Aa}{\mathcal{A}_{\overline{\alpha}}}

\DeclareCaptionStyle{ruled}{labelfont=normalfont,labelsep=colon,strut=off} 
\floatstyle{ruled}
\newfloat{listing}{tb}{lst}{}
\floatname{listing}{Listing}

\usepackage{booktabs}

\title{Contrastive Explanations in Quantitative Bipolar Argumentation Frameworks}
\author{
    Written by AAAI Press Staff\textsuperscript{\rm 1}\thanks{With help from the AAAI Publications Committee.}\\
    AAAI Style Contributions by Peter Patel Schneider,
    Sunil Issar,\\
    J. Scott Penberthy,
    George Ferguson,
    Hans Guesgen,
    Francisco Cruz\equalcontrib\corresponding,
    Marc Pujol-Gonzalez\equalcontrib\corresponding
}
\affiliations{
    \textsuperscript{\rm 1}Association for the Advancement of Artificial Intelligence\\

    1101 Pennsylvania Ave, NW Suite 300\\
    Washington, DC 20004 USA\\
    proceedings-questions@aaai.org
}

\title{Contrastive Explanations in Quantitative Bipolar Argumentation Frameworks}
\author {
    Xiang Yin\textsuperscript{\rm 1},
    Nico Potyka\textsuperscript{\rm 2},
    Antonio Rago\textsuperscript{\rm 3},
    Francesca Toni\textsuperscript{\rm 1}
}
\affiliations {
    \textsuperscript{\rm 1}Imperial College London, UK\\
    \textsuperscript{\rm 2}Cardiff University, UK\\
    \textsuperscript{\rm 3}King's College London, UK\\

    \{x.yin20, f.toni\}@imperial.ac.uk, potykan@cardiff.ac.uk, antonio.rago@kcl.ac.uk
}

\begin{document}

\maketitle

\begin{abstract}
Argumentation frameworks are useful tools for representing and reasoning with information in a variety of settings, e.g. in supplementing AI models as they perform classification tasks, with a notable benefit of providing additional explainability.
In this paper, we introduce \emph{contrastive explanations} for Quantitative Bipolar Argumentation Frameworks (QBAFs), one such formalism. Unlike most existing explanations for QBAFs, which explain the reasoning outcome of a single argument of interest (i.e. a \emph{topic argument}), contrastive explanations explain the difference between two topic arguments. We introduce a general form of contrastive attribution functions (CAFs) and establish a set of general properties they should satisfy. We introduce CAFs based on removal, gradients and Shapley values, and study their properties. Finally, to illustrate contrastive explanations, we demonstrate their usefulness in healthcare and bias identification settings.
\end{abstract}


\section{Introduction}

Argumentation frameworks have recently emerged as useful tools for representing and reasoning with information in a range of settings, particularly those where explainability may otherwise be lacking, e.g. in image classification \cite{Ayoobi_25,Kori_25}, recommendation \cite{Rago_21,Rago_25} and claim verification \cite{Freedman_25,Zhu_25}.
Quantitative Bipolar Argumentation Frameworks (QBAFs) \cite{baroni2015automatic} are one such formal model for reasoning with conflicting and supporting information. 
A QBAF consists of a set of \emph{arguments}, \emph{attack} and \emph{support} relations among them, and a \emph{base score function} assigning each argument a prior strength.
To evaluate a QBAF, a gradual semantics (e.g., \cite{DF_QuAD_Antonio}) is applied to update the base scores by aggregating the influence of attackers and supporters, resulting in a final \emph{strength} for each argument. 
QBAFs are well suited to model decision-support tasks (e.g., \cite{cocarascu2019extracting,Rago_21}): the final strengths can serve as decision scores, while the explicit attack and support structure provides a transparent account of the reasoning process (e.g. \cite{cocarascu2019extracting}).
Figure~\ref{fig_zebra} shows a toy QBAF for an animal classification problem. The possible classes (\texttt{zebra}, \texttt{horse}, \texttt{tiger}), and the relevant features (\texttt{stripes}, \texttt{herbivore}) are represented as arguments. Support and attack relations indicate whether a feature supports or attacks a class.  All arguments are assigned a neutral base score of $0.5$, and the DF-QuAD semantics \cite{DF_QuAD_Antonio} is applied to compute the strengths of the class arguments.
Here, \texttt{zebra} is predicted as it has the highest strength of $0.875$, while \texttt{horse} and \texttt{tiger} both have strength $0.5$. 

\begin{figure}[t]
    \centering
    \includegraphics[width=0.8\linewidth]{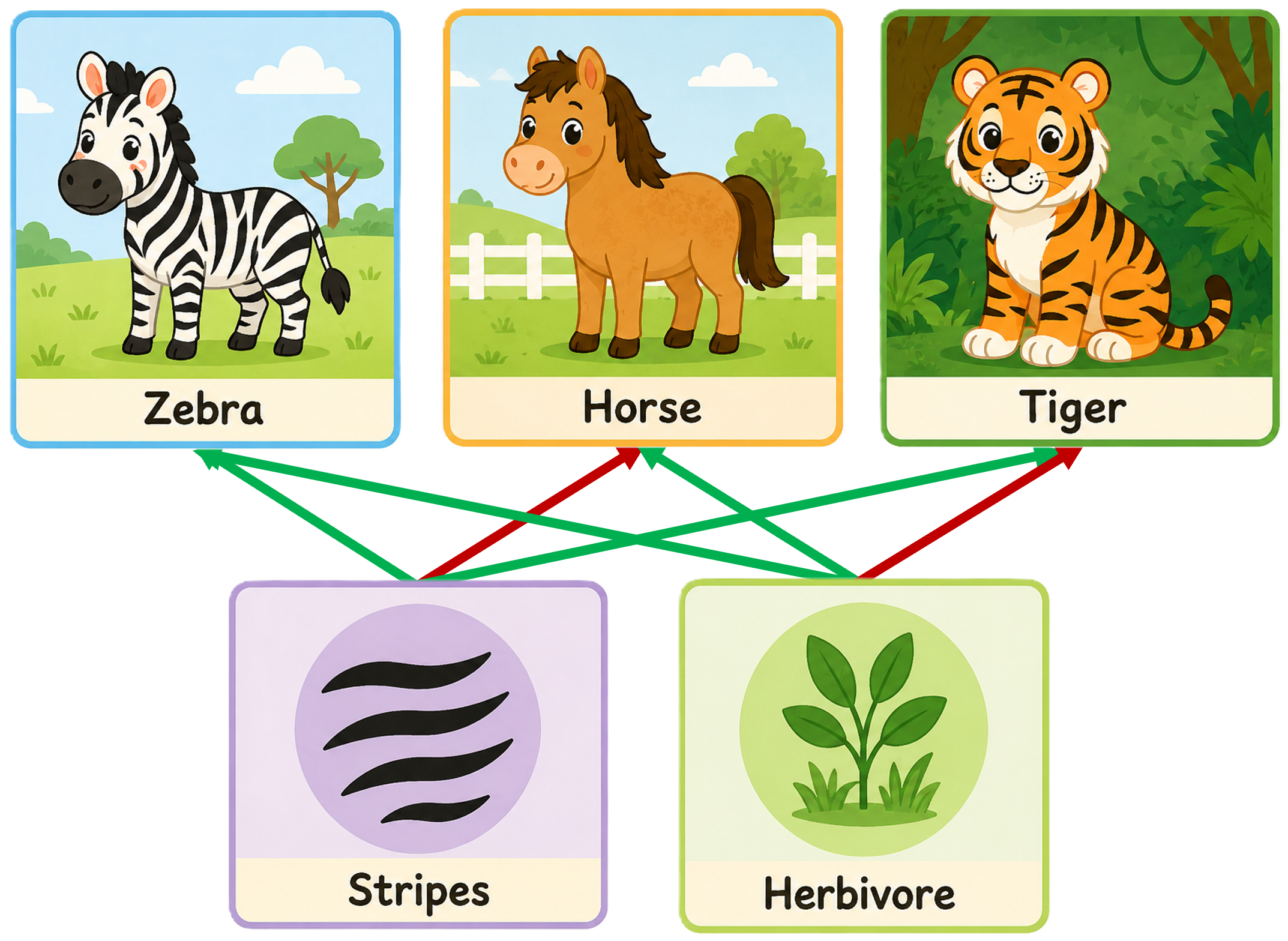}
    \caption{An example QBAF for an animal classification task.
    Squares represent arguments; green and red edges indicate support and attack relations, respectively. The contrastive explanation for zebra over horse is stripes, whereas that for zebra over tiger is being a herbivore.
    }
    \label{fig_zebra}
\end{figure}

Given this prediction, a user may ask: why \texttt{zebra}? A possible explanation is that the animal has \texttt{stripes} and is \texttt{herbivorous}. However, research in philosophy, psychology, and social science shows that explanations are contrastive: when people ask a ``Why $P$?'' question, they often implicitly ask ``why $P$ rather than $Q$?'', where $P$ is a \emph{fact} while $Q$ is some \emph{contrast case} \cite{lipton1990_philosophy,hilton1990_psychology,van2002_social_science,ylikoski2007idea,chin2017_cognitive,miller2019_XAI,miller2021contrastive}. In our example, if the question is why the prediction is \texttt{zebra} rather than \texttt{horse}, the \emph{contrastive explanation} should highlight \texttt{stripes}, since both \texttt{zebra} and \texttt{horse} are \texttt{herbivorous}. In contrast, if the question is why the prediction is \texttt{zebra} rather than \texttt{tiger}, the contrastive explanation should highlight \texttt{herbivore}, since this feature distinguishes \texttt{zebra} from \texttt{tiger} (in this example), whereas \texttt{stripes} does not. Thus, contrastive explanations aim to identify the discriminative reasons (referred to as \emph{difference conditions} in \cite{lipton1990_philosophy}) that distinguish $P$ from a contrast case $Q$, rather than listing all explanatory reasons. This makes contrastive explanations simpler, more feasible, and cognitively less demanding \cite{miller2021contrastive}.

However, existing explanation methods for QBAFs, such as attribution explanations \cite{YIN_AAE} and counterfactual explanations \cite{YIN_CEQ}, mainly focus on explaining the strength of a single argument of interest (\emph{topic argument}). In other words, they answer ``why $P$'' but overlook the contrast question ``why $P$ rather than $Q$'', and thus fail to capture the ``why not $Q$'' aspect. 
This motivates the central research question of this paper: 
    \textbf{Given two topic arguments in a QBAF, how can we explain the difference between their strengths?}

To address this question, we introduce attribution-based \emph{contrastive explanations} for QBAFs. 
Our key idea is to assign attribution scores to all non-topic arguments with respect to the strength gap between two topic arguments. 
Attribution scores are well suited to this purpose because they provide an intuitive quantitative measure of each argument's influence on the contrast. 
By comparing these scores, we can identify arguments that discriminate between the two topic arguments without manually inspecting all reasoning paths.
For example, when explaining why \texttt{zebra} rather than \texttt{horse}, \texttt{stripes} should receive a larger attribution score than \texttt{herbivore}, as the latter supports both classes.

The contributions of this paper are as follows:
\begin{itemize}
    \item We introduce contrastive explanations and desirable properties that they should satisfy (Sections~\ref{sec_contrastive_explanations}, \ref{sec_general_properties});
    \item We study three concrete contrastive explanation methods and their properties (Sections~\ref{sec_subtraction_CAFs}, \ref{sec_complexity}, \ref{sec_specific_properties});
    \item We illustrate the use of contrastive explanations in healthcare (Section~\ref{sec_case_study}) and bias detection (Section~\ref{sec_expr}).
\end{itemize}

\section{Preliminaries}
We recall the definition of QBAFs~\cite{baroni2015automatic}.
\begin{definition}[QBAF]
A \emph{QBAF}
is a quadruple $\mathcal{Q}=\left\langle\mathcal{A}, \mathcal{R}^{-}, \mathcal{R}^{+}, \tau \right\rangle$ where $\mathcal{A}$ is a finite set of \emph{arguments}; $\mathcal{R}^{-}, \mathcal{R}^{+}  \subseteq \mathcal{A} \times \mathcal{A}$  are \emph{attack} and \emph{support} relations
     such that $\mathcal{R}^{-} \cap \mathcal{R}^{+} = \emptyset$; $\tau: \mathcal{A} \rightarrow  [0,1]$ is a \emph{base score function}.
\end{definition}

In the remainder, unless specified otherwise, we will assume a QBAF $\mathcal{Q}=\left\langle\mathcal{A}, \mathcal{R}^{-}, \mathcal{R}^{+}, \tau \right\rangle$ as given. At several places, we will consider QBAF restrictions~\cite{kampik2024contribution} induced by subsets of arguments.

\begin{definition}[QBAF Restriction] 
    Given a subset of arguments $\mathcal{A}'\subseteq \mathcal{A}$, the QBAF restriction of $\mathcal{Q}$ to $\mathcal{A}'$ is $\mathcal{Q}_{\downarrow_{\mathcal{A}'}}=\left\langle\mathcal{A}', \mathcal{R}^{-}\cap(\mathcal{A}'\times \mathcal{A}'), \mathcal{R}^{+}\cap(\mathcal{A}'\times \mathcal{A}'), \tau \cap (\mathcal{A}'\times[0,1]) \right\rangle$.
\end{definition}

    
We use gradual semantics to assign a \emph{dialectical strength} to each argument in a QBAF. Examples of gradual semantics include DF-QuAD~\cite{DF_QuAD_Antonio}, Restricted Euler-based semantics (REB)~\cite{REB_Semantics}, and Quadratic Energy semantics (QE)~\cite{QE_Semantics}.

\begin{definition}[Gradual Semantics]
\label{def_semantics}
A \emph{gradual semantics} is a function 
$\sigma: \mathcal{A} \rightarrow  [0,1] \cup \{\bot\}$. We call
$\sigma(\alpha)$ the \emph{strength} of $\alpha$ and say that it is \emph{undefined} iff
$\sigma(\alpha) = \bot$.
\end{definition}
All gradual semantics that we are aware of are instances of the class of \emph{modular semantics} \cite{mossakowski2018} and we will make
use of some of their properties later. Roughly speaking, modular semantics compute strength values by an update process that starts
from the base scores, and repeatedly updates the strengths of arguments based on the strengths of their attackers and supporters.
The update function of modular semantics can be decomposed into an aggregation function that aggregates the strengths of attackers
and supporters, and an influence function that uses the aggregate to adapt the base score.

An undefined strength value ($\sigma(\alpha) = \bot$) may arise in some cyclic QBAFs when the update process fails to converge \cite{mossakowski2018}.
However, in all known cases, the convergence problems can be solved by continuizing the semantics \cite{Potyka19,PotykaB24}.
In the following, we will focus on QBAFs for which all strength values are defined.

To explain the strength of a \emph{topic argument} $\alpha\in\mathcal{A}$, attribution functions (AFs) $\phi_{\alpha}: \mathcal{A}\setminus\{\alpha\}\rightarrow \mathbb{R}$  quantify the \emph{influence} of arguments on $\alpha$ \cite{kampik2024contribution}.
In general, a larger magnitude of $\phi_{\alpha}(\beta)$ indicates a stronger influence, while its sign indicates whether the influence is \emph{positive} or \emph{negative}.

Removal-based AFs measure by how much the strength of $\alpha$ changes when $\beta$ is removed from $\mathcal{Q}$.
\begin{definition}[Removal-based AF]
\label{removal_aae}
For any $\alpha, \beta \in \mathcal{A}$ and $\alpha \neq \beta$, let $\mathcal{A}'=\mathcal{A}\setminus\{\beta\}$,
the \emph{removal-based attribution} from $\beta$ to $\alpha$ is $\phi_{\alpha}^{R}(\beta)=\sigma_{\mathcal{Q}}(\alpha)-\sigma_{\mathcal{Q}_{\downarrow_{\mathcal{A}'}}}(\alpha)$.
\end{definition}

Gradient-based AFs measure the sensitivity of $\alpha$ with respect to small changes in the base score of $\beta$.
\begin{definition}[Gradient-based AF]
\label{gradient_aae}
For all $\alpha, \beta \in \mathcal{A}$, $\alpha \neq \beta$, and \(\epsilon\in[-\tau(\beta),0)\cup(0,1-\tau(\beta)]\),
let $\mathcal{Q}_\epsilon=\left\langle\mathcal{A}, \mathcal{R}^{-}, \mathcal{R}^{+}, \tau_\epsilon \right\rangle$, where $\tau_\epsilon(\beta)=\tau(\beta)+\epsilon$ and $\tau_\epsilon(\gamma)=\tau(\gamma)$ for all $\gamma \in \mathcal{A}\setminus\{\beta\}$.
The \emph{gradient-based attribution} from $\beta$ to $\alpha$ is defined as $\phi_{\alpha}^{G}(\beta)= \lim_{\epsilon \to 0}\frac{\sigma_{\mathcal{Q}_\epsilon}(\alpha)-\sigma_{\mathcal{Q}}(\alpha)}{\epsilon}$.
\end{definition}

Shapley-based AFs measure the average contribution of $\beta$ to the strength of $\alpha$ when adding $\beta$ to a subgraph of $\mathcal{Q}$.
\begin{definition}[Shapley-based AF]
\label{shapley_aae}
For $\alpha, \beta \in \mathcal{A}$ and $\alpha \neq \beta$, the \emph{Shapley-based attribution} from $\beta$ to $\alpha$ is defined as
\begin{equation*}
\phi_{\alpha}^{S}(\beta)
=
\sum_{B \subseteq \Aa \setminus \{ \beta \}}
w(B) \cdot c_\beta(B),
\end{equation*}
where $\Aa = \mathcal{A} \setminus \{\alpha\}$, $w(B) = \frac{
|B|! \cdot 
\left(|\Aa|-|B|-1\right)!
}{
|\Aa|!
}$ and
$c_\beta(B) = 
\sigma_{\mathcal{Q}_{\downarrow_{B \cup \{\alpha, \beta\}}}}(\alpha)
-
\sigma_{\mathcal{Q}_{\downarrow_{B \cup \{ \alpha \}}}}(\alpha).
$
\end{definition}


\section{Contrastive Explanations}
\label{sec_contrastive_explanations}
To explain the strength difference between two topic arguments $\alpha$ and $\beta$, we introduce \emph{contrastive attribution functions (CAFs)}, which quantify the influence of arguments on their relative strength.

\begin{notation}
For any $\alpha,\beta\in\mathcal{A}$, we write
$\alpha\succeq\beta$ to denote the contrast $\alpha$ is stronger than $\beta$.
\end{notation}

\begin{definition}[Contrastive Attribution Function (CAF)]
A CAF for $\alpha,\beta \in \mathcal{A}$ is a function $\Phi_{\alpha\succeq\beta}: \mathcal{A}\setminus\{\alpha,\beta\}\rightarrow\mathbb{R}$.
\end{definition}

Intuitively, a \emph{positive influence} $\Phi_{\alpha\succeq\beta}(\gamma)>0$ means that $\gamma$ contributes more favourably to $\alpha$ than to $\beta$. This may occur when $\gamma$ supports $\alpha$ more strongly than $\beta$, attacks $\alpha$ less strongly than $\beta$, or supports $\alpha$ while attacking $\beta$. Conversely, a \emph{negative influence} means that $\gamma$ contributes more favourably to $\beta$ relative to $\alpha$.
If $\Phi_{\alpha\succeq\beta}(\gamma)=0$, $\gamma$ contributes equally to $\alpha$ and $\beta$ and we call the influence \emph{neutral}.

To narrow down the choice of a  CAF, we suggest some desirable properties that a CAF should satisfy.

\section{Desirable Properties of CAFs}
\label{sec_general_properties}
To begin with, if an argument positively affects $\alpha \succeq \beta$,
then it should negatively affect $\beta \succeq \alpha $ by the same amount.
\begin{property}[Antisymmetry]
For any $\alpha, \beta \in \mathcal{A}$ and $\gamma \in \mathcal{A} \setminus \{\alpha, \beta\}$, $\Phi_{\alpha\succeq\beta}(\gamma)=-\Phi_{\beta\succeq\alpha}(\gamma)$.
\end{property}

\begin{proposition}
If $\Phi$ satisfies antisymmetry, $\!$then $\Phi_{\alpha\succeq\alpha}(\gamma)\!\!=\!\!0$.
\end{proposition}
If an attribution function $\phi$ is well suited to measure the influence of arguments in a domain, we may want that a CAF is calibrated w.r.t. $\phi$ 
in the following sense.
\begin{property}[$\phi$-Calibration]
Let $\phi$ be an attribution function.
For any $\alpha, \beta \in \mathcal{A}$ and $\gamma \in \mathcal{A} \setminus \{\alpha, \beta\}$, 
if $\phi_{\alpha}(\gamma)= a$ and $\phi_{\beta}(\gamma) = 0$, then $\Phi_{\alpha\succeq\beta}(\gamma)= a$.
\end{property}
Let us note that antisymmetry implies that a symmetrical calibration property holds for the second argument.
\begin{proposition}[Inverse $\phi$-Calibration]
\label{prop_inverse_calibration}
If $\Phi$ satisfies antisymmetry and $\phi$-Calibration, then 
$\phi_{\alpha}(\gamma)= 0$ and $\phi_{\beta}(\gamma) = b$ implies $\Phi_{\alpha\succeq\beta}(\gamma)= -b$.
\end{proposition}

If the influence of $\gamma$ on both $\alpha$ and $\beta$ is 
equal according to a reference attribution function $\phi$, then the influence of $\gamma$ on $\alpha\succeq\beta$ should be $0$.

\begin{property}[$\phi$-Neutrality]
Let $\phi$ be an attribution function.
For any $\alpha, \beta \in \mathcal{A}$ and $\gamma \in \mathcal{A} \setminus \{\alpha, \beta\}$
if $\phi_{\alpha}(\gamma)=\phi_{\beta}(\gamma)$, then $\Phi_{\alpha\succeq\beta}(\gamma)=0$.
\end{property}
Similarly, if $\gamma$ influences $\alpha$ stronger than $\beta$ according to a reference attribution function $\phi$,
then the influence of $\gamma$ on $\alpha\succeq\beta$ should be positive. 
\begin{property}[$\phi$-Monotonicity]
Let $\phi$ be an attribution function.
For any $\alpha, \beta \in \mathcal{A}$ and $\gamma \in \mathcal{A} \setminus \{\alpha, \beta\}$, if $\phi_{\alpha}(\gamma)>\phi_{\beta}(\gamma)$, then $\Phi_{\alpha\succeq\beta}(\gamma)>0$.
\end{property}
Antisymmetry guarantees again a symmetric behaviour for the case that $\gamma$ influences $\alpha$ less  than $\beta$.
\begin{proposition}
If $\Phi$ satisfies antisymmetry and $\phi$-Monotonicity, then 
$\phi_{\alpha}(\gamma) < \phi_{\beta}(\gamma)$ implies $\Phi_{\alpha\succeq\beta}(\gamma) < 0$.    
\end{proposition}

\section{Derived CAFs}
\label{sec_subtraction_CAFs}

Since attribution functions measure the influence of arguments on an individual argument, one natural idea to define CAFs is to combine the individual attribution values
by a binary function.
\begin{definition}[Derived CAFs]
\label{def_derived_CAF}
A CAF $\Phi_{\alpha\succeq\beta}$ is called derived from an attribution function $\phi$ if there is a function 
$f: \mathbb{R}^2 \rightarrow \mathbb{R}$ such that 
for any $\gamma\in\mathcal{A}\setminus\{\alpha,\beta\}$,
$\Phi_{\alpha\succeq\beta}(\gamma) = f(\phi_{\alpha}(\gamma), \phi_{\beta}(\gamma))$.
\end{definition}
As we show next, the properties from the previous section can be satisfied by
inducing a CAF from a reference attribution function using subtraction.
\begin{proposition}
If $\Phi_{\alpha\succeq\beta}$ is derived from an attribution function $\phi$ using subtraction, 
then $\Phi_{\alpha\succeq\beta}$ satisfies 
Antisymmetry,
$\phi$-Calibration,
$\phi$-Neutrality and
$\phi$-Monotonicity.
\end{proposition}
While subtraction-derived CAFs satisfy all previously proposed properties, there could still be
other functions that lead to the same properties. We can characterise subtraction-derived measures
by adding the following additivity property.
\begin{property}[Additivity]
For any $\alpha, \beta, \eta \in \mathcal{A}$ and any $\gamma \in \mathcal{A}\setminus\{\alpha,\beta,\eta\}$, $\Phi_{\alpha\succeq\beta}(\gamma)=\Phi_{\alpha\succeq\eta}(\gamma)+\Phi_{\eta\succeq\beta}(\gamma)$.
\end{property}
\begin{proposition}
If $\Phi_{\alpha\succeq\beta}$ is derived using subtraction, then $\Phi_{\alpha\succeq\beta}$ satisfies additivity.
\end{proposition}
Additivity can be used to fully characterise subtraction-derived CAFs for all \emph{plausible} attribution functions. 
By plausible we mean the following: all gradual semantics that we are aware of belong to the class of \emph{modular semantics} 
and all modular semantics satisfy the \emph{independence} property which states that arguments can only affect each other
if there is a directed path between them \cite{PotykaB24b}. Consequently, when adding a new argument to a graph without connecting it to any other arguments,
the attribution value of this argument should be $0$ for all other arguments.
\begin{definition}[Plausibility]
An attribution function $\phi$ is called \emph{plausible} if for each QBAF $\mathcal{Q}$ and the QBAF $\mathcal{Q}'$ resulting from $\mathcal{Q}$
by adding a single argument $\eta$ (and no edges), it holds that 
(1) the attribution values  of all arguments from $\mathcal{Q}$ are equal in both
$\mathcal{Q}$ and $\mathcal{Q}'$, and
(2) $\phi_{\alpha}(\eta) = 0$ for all arguments $\alpha$ from $\mathcal{Q}$.
\end{definition}
All AFs introduced previously are plausible.
\begin{proposition}
$\phi_{\alpha}^{R}, \phi_{\alpha}^{G}$ and $\phi_{\alpha}^{S}$ are plausible AFs under all modular semantics.   
\end{proposition}
We have the following characterisation of subtraction-derived
measures.
\begin{proposition}
If $\Phi$ is a CAF derived from a plausible attribution function $\phi$, and $\Phi$ satisfies antisymmetry, $\phi$-calibration
and additivity, then, under all modular gradual semantics, $\Phi$ is equal to the 
CAF derived from $\phi$ using subtraction.
\end{proposition}
Note that, based on the choice of $\phi$, the derived CAF will be calibrated
differently, and so CAFs derived from different attribution functions will usually
lead to different CAFs.
In the following proposition, we present compact formulas for CAFs derived from our previously introduced AFs.
\begin{proposition}
The CAFs 
$\Phi_{\alpha\succeq\beta}^{R}, \Phi_{\alpha\succeq\beta}^{G}, \Phi_{\alpha\succeq\beta}^{S}$
derived from the removal-based, gradient-based and Shapley-based AFs using subtraction are defined as follows:
$$\Phi_{\alpha\succeq\beta}^{R}(\gamma)=\big(\sigma_{\mathcal{Q}}(\alpha)-\sigma_{\mathcal{Q}}(\beta)\big)- \big(\sigma_{\mathcal{Q}_{\downarrow_{\mathcal{A}'}}}(\alpha)-\sigma_{\mathcal{Q}_{\downarrow_{\mathcal{A}'}}}(\beta)\big).$$
$$\Phi_{\alpha\succeq\beta}^{G}(\gamma)= \lim_{\epsilon \to 0}\frac{\big(\sigma_{\mathcal{Q}'}(\alpha)-\sigma_{\mathcal{Q}'}(\beta)\big)- \big(\sigma_{\mathcal{Q}}(\alpha)-\sigma_{\mathcal{Q}}(\beta)\big)}{\epsilon}.$$
\begin{align*}
&\Phi_{\alpha\succeq\beta}^{S}(\gamma)
= \sum_{B \subseteq \A \setminus \{ \alpha, \beta, \gamma \}} \big(
w(B) \cdot ( c_{\gamma \rightarrow \alpha}(B) - c_{\gamma \rightarrow \beta}(B)) \\
&\ \ +
w(B \cup \{\beta\}) \cdot ( c_{\gamma \rightarrow \alpha}(B \cup \{\beta\}) - c_{\gamma \rightarrow \beta}(B \cup \{\alpha\}))
\big),    
\end{align*}
where $c_{\gamma \rightarrow x}(X) = 
\sigma_{\mathcal{Q}_{\downarrow_{X \cup \{x, \gamma\}}}}(x)
-
\sigma_{\mathcal{Q}_{\downarrow_{X \cup \{ x \}}}}(x)
$.
\end{proposition}

\section{Computing Subtraction-Derived CAFs}
\label{sec_complexity}
In applications, our topic arguments often correspond to alternatives that we can choose from.
For example, Figure \ref{fig_case_study_1} in Section \ref{sec_case_study} shows a medical decision QBAF where we can choose one of three different treatments
(antiviral, antibiotic or bronchodilator). Now consider a problem with $T$ topic arguments $t_1, \dots, t_T$. When computing the impact of an argument on
all preferences $t_i \succeq t_j$ naively, we require $T \cdot (T-1) = O(T^2)$ CAF calls. In applications like healthcare, where we may have dozens of different diagnoses
or treatments, this can be too expensive. We can exploit properties of subtraction-derived CAFs to reduce the number of CAF calls significantly.

To begin with, note that anti-symmetry allows us to compute $\Phi_{t_i \succeq t_j}(\gamma)$ from $\Phi_{t_j \succeq t_i}(\gamma)$.
Hence, it suffices to consider only preferences $t_i \succeq t_j$ with $i < j$. While this reduces the number of computations to
$\frac{T \cdot (T-1)}{2} = O(T^2)$, it remains quadratic asymptotically.

Additivity allows us to design a dynamic programming algorithm that requires only a linear number of CAF calls.
\begin{algorithm}[t]
\caption{Dynamic Programming CAF Computation}
\label{algo_caf}
\textbf{Input}: A QBAF $\mathcal{Q}$, 
CAF $\Phi_{\alpha \succeq \beta}$,
set of topic arguments $\{t_1, \dots, t_T\}$,
query argument $\gamma$.\\
\textbf{Output}: Map $M$ such that $M[i,j] = \Phi_{t_i \succeq t_j}(\gamma)$ for all $1 \leq i < j \leq T$.\\
\begin{algorithmic}[1] 
\STATE Initialise empty map $M$
\FOR {$i=1$ \texttt{to} $T-1$}
    \STATE $M[i,i+1] \leftarrow \Phi_{t_i \succeq t_{i+1}}(\gamma)$
\ENDFOR
\FOR {$i=1$ \texttt{to} $T-2$}
    \FOR {$j=i+2$ \texttt{to} $T$}
        \STATE $M[i,j] = M[i,j-1] + M[j-1, j]$.
    \ENDFOR
\ENDFOR
\STATE \textbf{return} $M$ \hfill 
\end{algorithmic}
\end{algorithm}

\begin{proposition}
If $\Phi_{\alpha \succeq \beta}$ satisfies additivity,
Algorithm \ref{algo_caf} computes $\Phi_{t_i \succeq t_j}(\gamma)$ for all $1 \leq i < j \leq T$
with $O(T)$  $\Phi_{\alpha \succeq \beta}$-calls.
If $\Phi_{\alpha \succeq \beta}$ can be computed in time $O(C)$,
the overall time complexity of the algorithm
is $O(T\cdot C + T^2)$.
\end{proposition}

\section{Method-Specific Properties}
\label{sec_specific_properties}

We regard the properties introduced in Section \ref{sec_general_properties} as desirable for all CAFs and they are indeed satisfied by all previously introduced CAFs. 
To distinguish our CAFs axiomatically, we introduce some additional properties inspired by~\cite{kampik2024contribution} to separate them.

\emph{Counterfactuality} captures the intuition that removing an argument with positive (negative) attribution should decrease (increase) the strength difference between the topic arguments.

\begin{property}[Counterfactuality]
$\Phi_{\alpha\succeq\beta}(\gamma)$ satisfies \emph{Counterfactuality}
iff, for any $\gamma\in\mathcal{A}\setminus\{\alpha, \beta\}$, letting $\mathcal{A}'=\mathcal{A} \setminus \{\gamma\}$,
the following statements hold:\\
1. If $\Phi_{\alpha\succeq\beta}(\gamma)<0$, then $\sigma_{\mathcal{Q}}(\alpha)-\sigma_{\mathcal{Q}}(\beta)<\sigma_{\mathcal{Q}_{\downarrow_{\mathcal{A}'}}}(\alpha)-\sigma_{\mathcal{Q}_{\downarrow_{\mathcal{A}'}}}(\beta)$;\\
2. If $\Phi_{\alpha\succeq\beta}(\gamma)>0$, then $\sigma_{\mathcal{Q}}(\alpha)-\sigma_{\mathcal{Q}}(\beta)>\sigma_{\mathcal{Q}_{\downarrow_{\mathcal{A}'}}}(\alpha)-\sigma_{\mathcal{Q}_{\downarrow_{\mathcal{A}'}}}(\beta)$.
\end{property}
\begin{proposition}
$\Phi_{\alpha\succeq\beta}^{R}(\gamma)$ satisfies Counterfactuality, while $\Phi_{\alpha\succeq\beta}^{G}(\gamma)$ and $\Phi_{\alpha\succeq\beta}^{S}(\gamma)$ can violate Counterfactuality.
\end{proposition}

Local faithfulness constrains the influence of an argument on the strength gap between two topic arguments when 
perturbing its base score. An argument with positive (negative) influence will locally increase (decrease) the strength gap between two topic arguments when its base score is slightly increased.
\begin{property}[Local Faithfulness]
\label{property_faithfulness}
$\Phi_{\alpha\succeq\beta}(\gamma)$ satisfies \emph{Local Faithfulness} wrt. $\sigma$ iff, for any $\gamma \in \mathcal{A}\setminus\{\alpha,\beta\}$,
there exists $\delta>0$ such that, for all $e \in [\tau(\gamma)-\delta, \tau(\gamma)+\delta] \cap [0,1]$, 
letting $\mathcal{Q}'=\left\langle\mathcal{A}, \mathcal{R}^{-}, \mathcal{R}^{+}, \tau' \right\rangle$ be the QBAF such that $\tau'(\gamma)=e$ and $\tau'(\eta)=\tau(\eta)$ for all $\eta \in \mathcal{A}\setminus\{\gamma\}$.
the following statements hold:\\
1. If $\Phi_{\alpha\succeq\beta}(\gamma)<0$, then $\sigma_{\mathcal{Q}}(\alpha)-\sigma_{\mathcal{Q}}(\beta)\leq\sigma_{\mathcal{Q}'}(\alpha)-\sigma_{\mathcal{Q}'}(\beta)$ whenever $e<\tau(\gamma)$, and $\sigma_{\mathcal{Q}}(\alpha)-\sigma_{\mathcal{Q}}(\beta)\geq\sigma_{\mathcal{Q}'}(\alpha)-\sigma_{\mathcal{Q}'}(\beta)$ whenever $e>\tau(\gamma)$;\\
2. If $\Phi_{\alpha\succeq\beta}(\gamma)>0$, then $\sigma_{\mathcal{Q}}(\alpha)-\sigma_{\mathcal{Q}}(\beta)\geq\sigma_{\mathcal{Q}'}(\alpha)-\sigma_{\mathcal{Q}'}(\beta)$ whenever $e<\tau(\gamma)$, and $\sigma_{\mathcal{Q}}(\alpha)-\sigma_{\mathcal{Q}}(\beta)\leq\sigma_{\mathcal{Q}'}(\alpha)-\sigma_{\mathcal{Q}'}(\beta)$ whenever $e>\tau(\gamma)$.
\end{property}
\begin{proposition}
$\Phi_{\alpha\succeq\beta}^{G}(\gamma)$ satisfies Local Faithfulness, while $\Phi_{\alpha\succeq\beta}^{R}(\gamma)$ and $\Phi_{\alpha\succeq\beta}^{S}(\gamma)$ can violate Local Faithfulness.
\end{proposition}

\emph{Cross-topic-adjusted Efficiency} states that the overall contrastive effect should be fully explained by two components: the aggregate attribution of the non-topic arguments and a correction capturing the imbalance between the attributions of the two topic arguments to each other. The sum of these components should correspond to the difference between the final-strength gap and the base-score gap of the two topic arguments. When the cross-topic attributions are equal, the correction vanishes and the property reduces to standard Efficiency \cite{shapley1951notes}.
\begin{property}[Cross-topic-adjusted Efficiency]
\label{property_efficiency}
$\Phi_{\alpha\succeq\beta}(\gamma)$ satisfies \emph{Cross-topic-adjusted Efficiency} iff \(\sum_{\gamma\in\mathcal A\setminus\{\alpha,\beta\}}\Phi_{\alpha\succeq\beta}(\gamma)+\bigl(\phi_{\alpha}(\beta)-\phi_{\beta}(\alpha)\bigr)=\bigl(\sigma_{\mathcal Q}(\alpha)-\sigma_{\mathcal Q}(\beta)\bigr)-\bigl(\tau(\alpha)-\tau(\beta)\bigr)\).
\end{property}
\begin{proposition}
\label{prop_method_efficiency}
$\Phi_{\alpha\succeq\beta}^{S}(\gamma)$ satisfies Cross-topic-adjusted Efficiency, while $\Phi_{\alpha\succeq\beta}^{R}(\gamma)$ and $\Phi_{\alpha\succeq\beta}^{G}(\gamma)$ can violate it.
\end{proposition}




\paragraph{Discussion}

While we do not regard Counterfactuality, Local Faithfulness and Cross-topic-adjusted Efficiency as essential, we showed that they allow us to distinguish our CAFs and can therefore guide the choice of the CAF in practice.
Intuitively, the removal-based CAF measures how the strength difference changes when an argument is removed and is therefore appropriate when a counterfactual interpretation of attribution values is desirable. 
The gradient-based CAF captures the local sensitivity of the strength difference to changes in arguments' base scores, and can be used when we are interested in the robustness to small changes. 
Finally, the Shapley-based CAF measures an argument's average marginal contribution and is preferable when the difference should be fully explained by the attribution values and the cross-topic component.

\section{CAFs for Healthcare}
\label{sec_case_study}




\begin{figure}[t]
    \centering
    \includegraphics[width=1.0\linewidth]{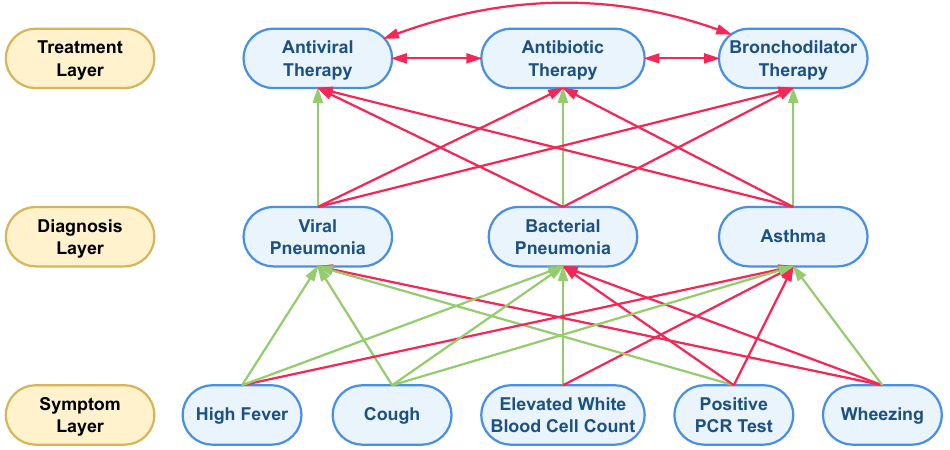}
    \caption{{A QBAF for treatment recommendation in healthcare (taken from \cite{EAI_needs_Argumentation}). Blue nodes denote arguments and green/red edges, resp., support/attack relations.}}
    \label{fig_case_study_1}
\end{figure}

\begin{figure*}[t]
    \centering
    \includegraphics[width=1.0\linewidth]{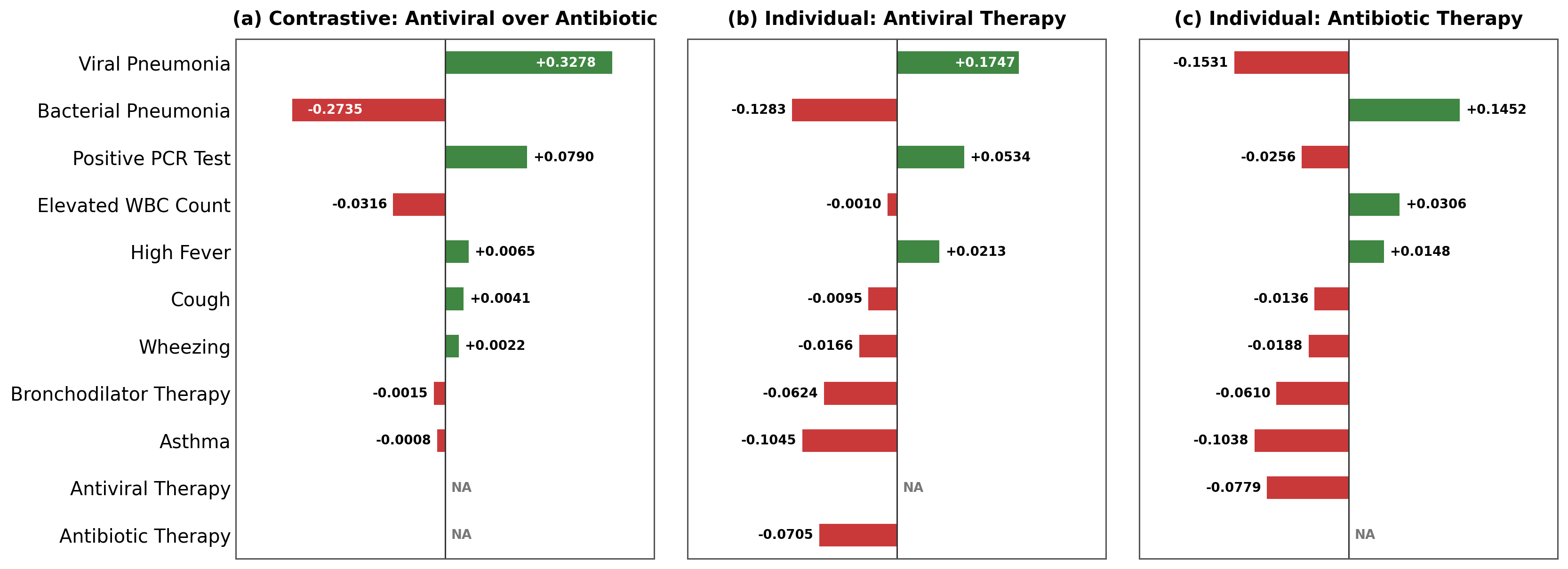}
    \caption{Shapley-based contrastive and individual explanations for treatment selection. Green bars indicate positive influence, while red bars indicate negative influence.}
    \label{fig_case_study_111}
\end{figure*}

We now present an illustrative example of CAFs in a healthcare setting. Figure~\ref{fig_case_study_1} (taken from \cite{EAI_needs_Argumentation}) shows how a QBAF can support practitioners in selecting suitable treatments for a patient based on observed symptoms~\footnote{This is a simplified, non-clinically validated example.}.
This QBAF consists of three layers. The bottom layer represents observable symptoms (e.g., \emph{cough}, \emph{wheezing}). The intermediate layer represents possible diagnoses (e.g., \emph{viral/bacterial pneumonia}). The top layer represents possible treatments (e.g., \emph{antiviral/antibiotic therapy}).
These layers are connected hierarchically via \emph{attack} and \emph{support} relations, which capture relationships between symptoms, diagnoses, and treatments. For example, \emph{wheezing} is commonly associated with \emph{asthma} (support). \emph{Asthma} can be treated by \emph{bronchodilator therapy} (support) but not by \emph{antiviral therapy} (attack). Treatment arguments are connected by mutual attacks because they correspond to competing treatment choices.

Base scores can reflect prior information such as the severity of observed symptoms, the plausibility of diagnoses based on a patient's medical history, or the suitability of candidate treatments. For instance, a higher body temperature may induce a higher base score for the \emph{high fever} argument. In this example, we assign a base score of $0.5$ to all arguments.

After specifying the QBAF structure and  base scores, we compute the strengths of arguments under QE semantics~\cite{QE_Semantics}, following the setting of \cite{EAI_needs_Argumentation}. Since the symptom arguments have neither attackers nor supporters, their strengths remain equal to their base scores, namely $0.5$. The strengths of the diagnosis arguments are $0.75$ for \emph{viral pneumonia}, $0.6$ for \emph{bacterial pneumonia}, and $0.4$ for \emph{asthma}.  The strengths of the treatment arguments are $0.3564$ for \emph{antiviral therapy}, $0.2368$ for \emph{antibiotic therapy}, and $0.1479$ for \emph{bronchodilator therapy}. Thus, \emph{antiviral therapy} is selected as the most suitable treatment, as it has the highest strength among the candidate treatments.

The strength ranking identifies the recommended treatment, but does not explain which symptoms are responsible for this choice. For example, a patient may ask why \emph{antiviral therapy} is recommended instead of \emph{antibiotic therapy}. This is a contrastive question and we can use Shapley-based CAFs to compute an answer.
We present gradient-based explanations in Section \ref{sec_additional_attributions_appendix} of the Supplementary Material.

\paragraph{Contrastive Explanations.}
Figure~\ref{fig_case_study_111}(a) shows the Shapley-based contrastive explanations for our example. 
We first analyse the attributions for the five symptom arguments. Among them, \emph{Positive PCR Test} (abbreviated as \emph{PCR}) has the largest positive influence on the contrast of the selected treatment (\emph{antiviral therapy}) over the alternative (\emph{antibiotic therapy}). 
The transparent QBAF structure allows us to visualise this influence through paths connecting \emph{PCR} to the two treatment arguments. Through \emph{viral pneumonia}, \emph{PCR} supports a diagnosis that supports \emph{antiviral therapy} and attacks \emph{antibiotic therapy}. Thus, these paths both strengthen the selected treatment and weaken the alternative. Furthermore, through \emph{bacterial pneumonia}, \emph{PCR} attacks a diagnosis that supports \emph{antibiotic therapy} and attacks \emph{antiviral therapy}. 
The paths through \emph{Asthma} are less discriminative for this contrast, since weakening \emph{Asthma} benefits both treatments.


We next analyse the influence of arguments in the diagnosis layer.  \emph{Viral} and \emph{bacterial pneumonia} have the strongest positive and negative influence on the contrast, respectively. This is intuitive because \emph{viral pneumonia} directly supports the selected treatment and attacks the alternative; whereas \emph{bacterial pneumonia} directly attacks the selected treatment and supports the alternative.

\paragraph{Individual Explanations.}
To clarify what we gain by using CAFs, we compare them to individual AFs. This comparison highlights several notable differences.
First, an argument that has equal or similar influences on two individual treatments may become much less influential in the contrastive explanation. For example, \emph{asthma} has a negative influence on both \emph{antiviral therapy} and \emph{antibiotic therapy} in the individual explanations. However, in the contrastive explanation, its attribution score is close to $0$. This indicates that although \emph{asthma} negatively influences the selected treatment, it is not a distinguishing argument in favour of \emph{antiviral therapy} over \emph{antibiotic therapy} because it also negatively influences the alternative in a symmetric manner.
Second, an argument that is not very influential on the selected treatment in the individual explanation may become influential in the contrastive explanation. For instance, \emph{elevated WBC count} has little influence on \emph{antiviral therapy}, but it has the largest negative attribution score in the contrastive explanation due to its strong positive influence on \emph{antibiotic therapy}.

\paragraph{Property Illustration.}
We illustrate \emph{cross-topic-adjusted efficiency}, a key property uniquely guaranteed by Shapley-based CAFs. 
For the contrastive explanation in Figure~\ref{fig_case_study_111}(a), the final strength difference between \emph{antiviral} and \emph{antibiotic therapy} arguments is $0.3564 - 0.2368 = 0.1196$, while the base score difference is $0.5-0.5=0$. 
The non-topic contrastive attributions sum to $0.1123$, while the cross-topic component is $-0.0705-(-0.0779)=0.0074$. Using unrounded values, their adjusted sum is $0.1196$, as required.
For the individual explanations, the attribution scores sum to $-0.1436$ for \emph{antiviral therapy} in Figure~\ref{fig_case_study_111}(b) and $-0.2632$ for \emph{antibiotic therapy} in Figure~\ref{fig_case_study_111}(c). These sums match the changes from the empty coalition to the full player set: $0.5 -0.1436 = 0.3564$ and $0.5 - 0.2632 = 0.2368$.

\section{CAFs for Bias Detection}
\label{sec_expr}
In this section, we demonstrate how contrastive explanations can help uncover biases in classification tasks \cite{jacovi2021contrastive}. We use gradient-based CAFs to detect bias in multilayer perceptrons (MLPs), as analysing local sensitivity through gradient-based methods is a natural and widely used approach to explaining differentiable neural networks (e.g., Integrated Gradients \cite{sundararajan2017axiomatic}). In addition, gradient-based CAFs are computationally efficient, as they avoid the combinatorial calculations required by Shapley-based CAFs.

\textbf{Dataset Selection and Preprocessing.} We use the COMPAS dataset \cite{propublica2016compas,angwin2016machine}, which is widely used in algorithmic fairness research. The dataset contains defendants' demographic and criminal-history information and provides assessments of their recidivism risk. After data cleaning and restricting the dataset to \texttt{African-American} and \texttt{Caucasian} individuals, $5,278$ instances remain. Each instance is represented by eight input features: age, counts of juvenile felonies, misdemeanours, and other offences, number of prior offences, current charge degree, sex, and race. The \texttt{score\_text} field is used as the classification label and defines three risk categories: \texttt{Low}, \texttt{Medium}, and \texttt{High}.

To train a classifier with a known bias, we modify the labels before splitting the dataset. Specifically, we randomly relabel 70\% of the \texttt{African-American} instances in the \texttt{Low} category as \texttt{Medium}, and around 70\% of those in the \texttt{Medium} category as \texttt{High}, while leaving the labels of all \texttt{Caucasian} instances unchanged. This controlled modification enables the classifier to learn a race-related bias.

\textbf{Classifier Architecture and Training.} We train an MLP with one hidden layer containing 16 neurons on the modified dataset. 
The training/validation/test split was 64/16/20 \%. The trained classifier achieves an accuracy of 64.58\% on the  test set. This is very low, but since we are only interested in explaining what a classifier learnt, it does not affect our experiments.

\begin{figure}[t]
    \centering
    \includegraphics[width=0.9\linewidth]{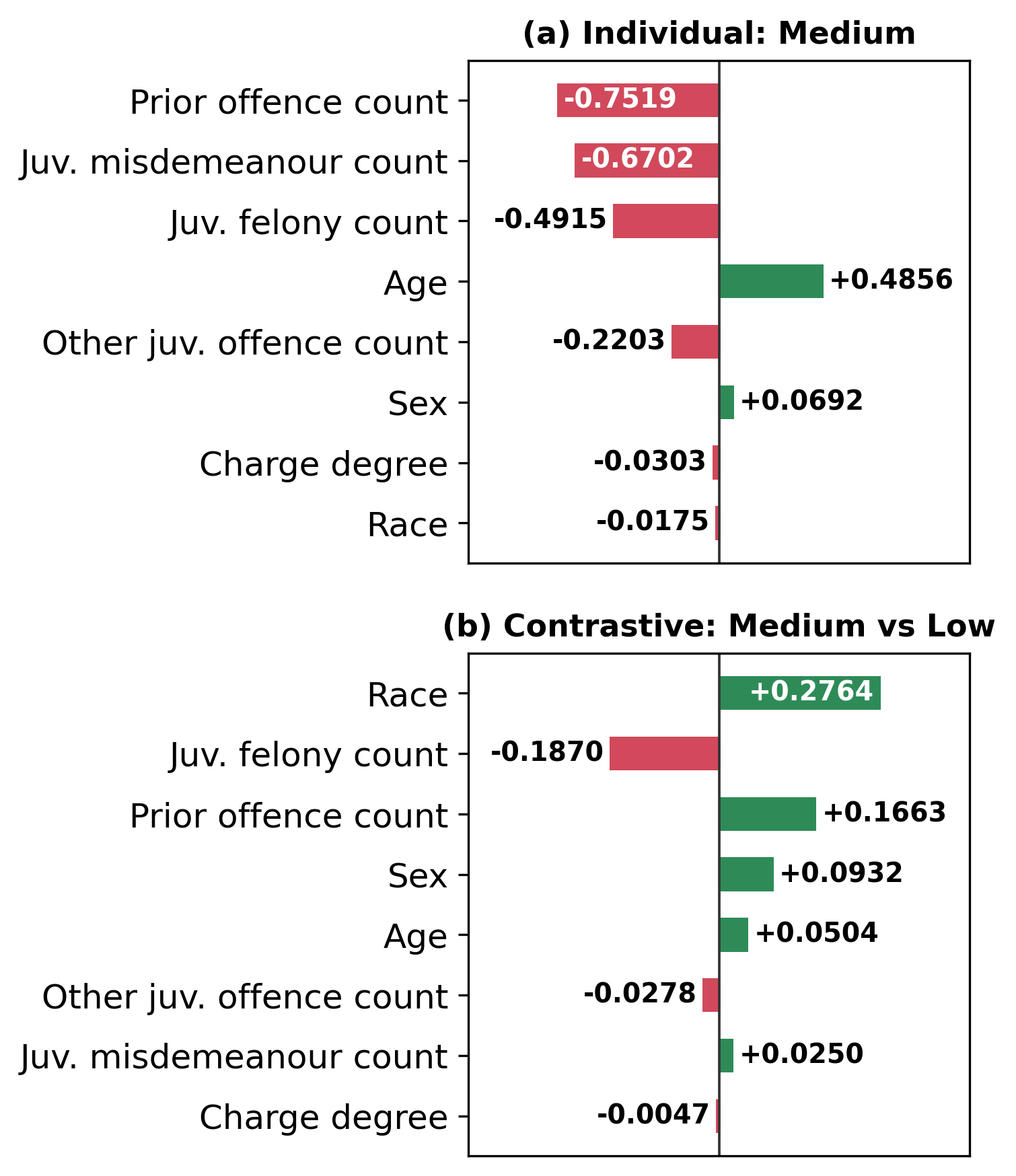}
    \caption{Gradient-based individual and contrastive AAEs for a defendant classified as \texttt{Medium} risk. The \texttt{race} argument is the least prominent in the individual explanation but becomes the most prominent when explaining why the defendant is classified as \texttt{Medium} rather than \texttt{Low} risk.}
    \label{fig_case_study_1111}
\end{figure}

\textbf{Explanations.} We select an \texttt{African-American} instance from the test set that is correctly classified as \texttt{Medium}. When its race is changed to \texttt{Caucasian} while all other input features are held fixed, the MLP's prediction changes from \texttt{Medium} to \texttt{Low}. Following the correspondence established in \cite{potyka2021interpreting}, we represent the MLP as a QBAF in which neurons are represented as arguments and weighted neural connections as support or attack relations
. In particular, the eight input neurons correspond to eight input arguments, while the three output neurons correspond to the output arguments \texttt{Low}, \texttt{Medium}, and \texttt{High}. We then apply gradient-based AAEs to quantify the attribution of each input argument to the strength of the output argument \texttt{Medium}.

As shown in Figure~\ref{fig_case_study_1111}(a), the gradient-based attribution from the input argument \texttt{race} to the output argument \texttt{Medium} is only $-0.0175$, the smallest in absolute magnitude among the eight input arguments. In contrast, when explaining why the prediction is \texttt{Medium} rather than \texttt{Low}, the contrastive gradient-based attribution of \texttt{race} is $0.2764$ and becomes the largest among all input arguments, as shown in Figure~\ref{fig_case_study_1111}(b). This indicates that \texttt{race} substantially increases the model's preference for \texttt{Medium} over \texttt{Low}, although this race-related bias is not prominent in the individual AAEs. Thus, in this example, contrastive explanations reveal the controlled race-related bias in the MLP's prediction.

To investigate whether this observation extends beyond our illustrative instance, we analysed $141$ instances that are correctly classified as \texttt{Medium} in the test set, but whose predictions change to \texttt{Low} when race alone is changed from \texttt{African-American} to \texttt{Caucasian}. For each instance, the eight input arguments are ranked according to the absolute values of their gradient-based attributions. The mean importance rank of the \texttt{race} argument improves from $5.16$ under the individual AAE to $2.62$ under the contrastive AAE (with mean absolute attributions of $0.1424$ and $0.4478$, respectively). Note that a rank closer to $1$ indicates a larger absolute attribution and hence greater importance. Hence, contrastive CAFs make the influence of the \texttt{race} argument more prominent than AFs.

\section{Related Work}
Our work is related to \cite{kampik2024change}, which considers a QBAF $\mathcal{Q}$ and its updated version $\mathcal{Q}'$, where the update may involve adding arguments, support/attack relations, or changing base scores of arguments. For any two topic arguments $\alpha$ and $\beta$ in both $\mathcal{Q}$ and $\mathcal{Q}'$, a \emph{strength inconsistency} occurs when the relative ordering of their strengths is not preserved after the update (e.g., $\sigma_{\mathcal{Q}}(\alpha)>\sigma_{\mathcal{Q}}(\beta)$ but $\sigma_{\mathcal{Q}'}(\alpha)<\sigma_{\mathcal{Q}'}(\beta)$).
They define three types of argument-set-based explanations: \emph{sufficient explanations}, whose changes alone can cause the inconsistency; \emph{counterfactual explanations}, which are sufficient explanations whose reversal restores strength consistency; and \emph{necessary explanations}, which intersect all sufficient explanations and therefore capture sets of changes that cannot all be avoided.
In contrast, our work focuses on explaining the relative strength difference between two topic arguments within a single QBAF, rather than explaining strength inconsistency induced by an update from one QBAF to another.

\cite{kampik2026strength} consider a QBAF with multiple topic arguments and studies how to obtain a different, often desirable, ordering over their strengths by modifying the base scores of arguments in the QBAF. This is different from our work in that we focus on explaining the strength difference between two topic arguments, rather than obtaining a different strength ordering.

Our work is related to a growing line of research on explaining the strength of an individual topic argument $\alpha \in \mathcal{A}$ in QBAFs and their edge-weighted variants. Existing methods can be broadly divided into attribution-based and counterfactual approaches.
Attribution-based methods aim to explain $\sigma_{\mathcal{Q}}(\alpha)$ by measuring the influence of other arguments on $\alpha$~\cite{delobelle2019interpretability,vcyras2022dispute,kampik2024contribution,YIN_AAE,Caren_impact_measure,Filip_Set_contribution}, or the influence of relations on $\alpha$~\cite{amgoud2017measuring,YIN_RAE}. 
By contrast, counterfactual explanations study how a desired strength value for $\alpha$ can be obtained, for instance by modifying argument base scores~\cite{YIN_CEQ} or, in edge-weighted QBAFs, by modifying edge weights~\cite{YIN_Contest_EWQBAF}. 
These works focus on explaining a single topic argument, whereas we focus on explaining the difference between two topic arguments. 


\section{Conclusions}
We introduced contrastive explanations for QBAFs to explain differences in the strengths of two topic arguments. As we saw, anti-symmetric CAFs that are calibrated with respect to an AF can be derived from the AF via subtraction.
Furthermore, we showed that if a CAF satisfies anti-symmetry, additivity and is calibrated with respect to a plausible AF, it must be equivalent to the CAF derived from the AF via subtraction. Additivity also allowed us to 
design a dynamic programming algorithm that requires only a linear number of CAF calls when computing attribution values for all combinations of topic arguments. We considered three instanstiations based on removal, gradients
and Shapley values and separated them by their distinctive properties. A case study in healthcare illustrates the practical applicability of our approach, and experiments demonstrate that contrastive explanations can help uncover bias in classifiers.

As future work, we plan to extend our approach to explain rankings among topic arguments. We also intend to conduct user studies to assess whether contrastive explanations improve users’ understanding of the mechanisms underlying QBAFs. Finally, we aim to explore further application domains in which contrastive explanations may be beneficial, e.g. in product recommendation \cite{Rago_25} where contrastive explanations seem like a natural fit.


\bibliography{aaai2027}

@book{shapley1951notes,
  title={Notes on the N-person Game},
  author={Shapley, Lloyd S},
  year={1951},
  optpublisher={Rand Corporation}
}

@inproceedings{Kori_25,
  author       = {Avinash Kori and
                  Antonio Rago and
                  Francesca Toni},
  editor       = {Sanmay Das and
                  Ann Now{\'{e}} and
                  Yevgeniy Vorobeychik},
  title        = {Free Argumentative Exchanges for Explaining Image Classifiers},
  booktitle    = {International Conference on Autonomous Agents
                  and Multiagent Systems (AAMAS)},
  pages        = {1172--1180},
  publisher    = {International Foundation for Autonomous Agents and Multiagent Systems
                  / {ACM}},
  year         = {2025}
}

@inproceedings{Freedman_25,
  author       = {Gabriel Freedman and
                  Adam Dejl and
                  Deniz Gorur and
                  Xiang Yin and
                  Antonio Rago and
                  Francesca Toni},
  editor       = {Toby Walsh and
                  Julie Shah and
                  Zico Kolter},
  title        = {Argumentative Large Language Models for Explainable and Contestable
                  Claim Verification},
  booktitle    = {{AAAI} Conference on Artificial Intelligence (AAAI)},
  pages        = {14930--14939},
  publisher    = {{AAAI} Press},
  year         = {2025},
  url          = {https://doi.org/10.1609/aaai.v39i14.33637},
  doi          = {10.1609/AAAI.V39I14.33637},
  bibsource    = {dblp computer science bibliography, https://dblp.org}
}

@inproceedings{Zhu_25,
  author       = {Yuqicheng Zhu and
                  Nico Potyka and
                  Daniel Hern{\'{a}}ndez and
                  Yuan He and
                  Zifeng Ding and
                  Bo Xiong and
                  Dongzhuoran Zhou and
                  Evgeny Kharlamov and
                  Steffen Staab},
  title        = {ArgRAG: Explainable Retrieval Augmented Generation using Quantitative
                  Bipolar Argumentation},
  booktitle    = {International Conference on Neurosymbolic
                  Learning and Reasoning (NeSy)},
  series       = {Proceedings of Machine Learning Research},
  volume       = {284},
  pages        = {697--718},
  publisher    = {{PMLR}},
  year         = {2025}
}

@article{Rago_21,
  author       = {Antonio Rago and
                  Oana Cocarascu and
                  Christos Bechlivanidis and
                  David A. Lagnado and
                  Francesca Toni},
  title        = {Argumentative explanations for interactive recommendations},
  journal      = {Artif. Intell.},
  volume       = {296},
  pages        = {103506},
  year         = {2021},
  url          = {https://doi.org/10.1016/j.artint.2021.103506},
  doi          = {10.1016/J.ARTINT.2021.103506},
  bibsource    = {dblp computer science bibliography, https://dblp.org}
}

@inproceedings{Ayoobi_25,
  author       = {Hamed Ayoobi and
                  Nico Potyka and
                  Francesca Toni},
  opteditor       = {Toby Walsh and
                  Julie Shah and
                  Zico Kolter},
  title        = {ProtoArgNet: Interpretable Image Classification with Super-Prototypes
                  and Argumentation},
  booktitle    = {{AAAI} Conference on Artificial Intelligence (AAAI)},
  pages        = {1791--1799},
  optpublisher    = {{AAAI} Press},
  year         = {2025}
}

@article{Rago_25,
  author       = {Antonio Rago and
                  Oana Cocarascu and
                  Joel Oksanen and
                  Francesca Toni},
  title        = {Argumentative review aggregation and dialogical explanations},
  journal      = {Artif. Intell.},
  volume       = {340},
  pages        = {104291},
  year         = {2025},
  url          = {https://doi.org/10.1016/j.artint.2025.104291},
  doi          = {10.1016/J.ARTINT.2025.104291},
  bibsource    = {dblp computer science bibliography, https://dblp.org}
}

@article{kampik2024change,
  title={Change in quantitative bipolar argumentation: Sufficient, necessary, and counterfactual explanations},
  author={Kampik, Timotheus and {\v{C}}yras, Kristijonas and Alarc{\'o}n, Jos{\'e} Ruiz},
  journal={International Journal of Approximate Reasoning},
  volume={164},
  pages={109066},
  year={2024},
  publisher={Elsevier}
}

@incollection{YIN_AAE,
  title={Argument Attribution Explanations in Quantitative Bipolar Argumentation Frameworks},
  author={Yin, Xiang and Potyka, Nico and Toni, Francesca},
  booktitle={ECAI 2023},
  pages={2898--2905},
  year={2023},
  publisher={IOS Press}
}

@inproceedings{YIN_RAE,
  title={Explaining arguments' strength: unveiling the role of attacks and supports},
  author={Yin, Xiang and Potyka, Nico and Toni, Francesca},
  booktitle={Proceedings of the Thirty-Third International Joint Conference on Artificial Intelligence},
  pages={3622--3630},
  year={2024}
}

@inproceedings{YIN_CEQ,
  title={CE-QArg: Counterfactual Explanations for Quantitative Bipolar Argumentation Frameworks},
  author={Yin, Xiang and Potyka, Nico and Toni, Francesca},
  booktitle={Proceedings of the International Conference on Principles of Knowledge Representation and Reasoning},
  volume={21},
  number={1},
  pages={697--707},
  year={2024}
}

@article{miller2021contrastive,
  title={Contrastive explanation: A structural-model approach},
  author={Miller, Tim},
  journal={The Knowledge Engineering Review},
  volume={36},
  pages={e14},
  year={2021},
  publisher={Cambridge University Press}
}

@inproceedings{EAI_needs_Argumentation,
  title={Towards an Argumentative Foundation for Evaluative AI},
  author={Yin, Xiang and Miller, Tim and Potyka, Nico and Rago, Antonio and Toni, Francesca},
  booktitle={Workshop on Explainable Artificial Intelligence (XAI) at IJCAI (To appear)},
  year={2026}
}

@inproceedings{YIN_Contest_EWQBAF,
    title     = {{Contestability in Edge-Weighted Quantitative Bipolar Argumentation Frameworks}},
    author    = {Yin, Xiang and Potyka, Nico and Rago, Antonio and Kampik, Timotheus and Toni, Francesca},
    booktitle = {{Proceedings of the 23rd International Conference on Principles of Knowledge Representation and Reasoning}},
    pages     = {676--687},
    year      = {2026},
    month     = {7},
    doi       = {10.24963/kr.2026/64},
    url       = {https://doi.org/10.24963/kr.2026/64},
  }

@article{Caren_impact_measure,
  title={Impact measures for gradual argumentation semantics},
  author={Anaissy, Caren Al and Delobelle, J{\'e}r{\^o}me and Vesic, Srdjan and Yun, Bruno},
  booktitle={International Conferences on Autonomous Agents and Multiagent Systems},
  year={2025}
}

@inproceedings{QE_Semantics,
  author    = {Nico Potyka},
  title     = {Continuous Dynamical Systems for Weighted Bipolar Argumentation},
  booktitle = {16th International Conference on Principles of Knowledge Representation and Reasoning (KR)},
  optpages     = {148--157},
  optpublisher = {{AAAI} Press},
  year      = {2018}
}

@article{REB_Semantics,
  title={Evaluation of arguments in weighted bipolar graphs},
  author={Amgoud, Leila and Ben-Naim, Jonathan},
  journal={International Journal of Approximate Reasoning},
  volume={99},
  pages={39--55},
  year={2018},
  publisher={Elsevier}
}

@inproceedings{DF_QuAD_Antonio,
  title={Discontinuity-free decision support with quantitative argumentation debates},
  author={Rago, Antonio and Toni, Francesca and Aurisicchio, Marco and Baroni, Pietro},
  booktitle={15th International Conference on the Principles of Knowledge Representation and Reasoning (KR)},
  year={2016}
}

@article{Filip_Set_contribution,
title = {Set contribution functions for quantitative bipolar argumentation and their principles},
journal = {International Journal of Approximate Reasoning},
volume = {194},
pages = {109673},
year = {2026},
issn = {0888-613X},
doi = {https://doi.org/10.1016/j.ijar.2026.109673},
url = {https://www.sciencedirect.com/science/article/pii/S0888613X26000496},
author = {Filip Naudot and Andreas Brännström and Vicenç Torra and Timotheus Kampik}
}

@inproceedings{delobelle2019interpretability,
  title={Interpretability of gradual semantics in abstract argumentation},
  author={Delobelle, J{\'e}r{\^o}me and Villata, Serena},
  booktitle={European Conference on Symbolic and Quantitative Approaches with Uncertainty},
  pages={27--38},
  year={2019},
  organization={Springer}
}

@inproceedings{vcyras2022dispute,
  title={Dispute Trees as Explanations in Quantitative (Bipolar) Argumentation},
  author={{\v{C}}yras, Kristijonas and Kampik, Timotheus and Weng, Qingtao},
  booktitle={ArgXAI 2022, 1st International Workshop on Argumentation for eXplainable AI, Cardiff, Wales, September 12, 2022},
  volume={3209},
  year={2022}
}

@inproceedings{amgoud2017measuring,
  title={Measuring the intensity of attacks in argumentation graphs with shapley value},
  author={Amgoud, Leila and Ben-Naim, Jonathan and Vesic, Srdjan},
  booktitle={International Joint Conference on Artificial Intelligence (IJCAI)},
  year={2017},
  pages        = {63--69}
}

@article{kampik2024contribution,
  title={Contribution functions for quantitative bipolar argumentation graphs: A principle-based analysis},
  author={Kampik, Timotheus and Potyka, Nico and Yin, Xiang and {\v{C}}yras, Kristijonas and Toni, Francesca},
  journal={International Journal of Approximate Reasoning},
  volume={173},
  pages={109255},
  year={2024},
  publisher={Elsevier}
}

@inproceedings{cocarascu2019extracting,
  title={Extracting dialogical explanations for review aggregations with argumentative dialogical agents},
  author={Cocarascu, Oana and Rago, Antonio and Toni, Francesca},
  booktitle={18th International Conference on Autonomous Agents and MultiAgent Systems (AAMAS)},
  optpages={1261--1269},
  year={2019},
  optorganization={Association for Computing Machinery}
}

@article{kampik2026strength,
  title={Strength change explanations in quantitative argumentation},
  author={Kampik, Timotheus and Yin, Xiang and Potyka, Nico and Toni, Francesca},
  year={2026},
  publisher={International Conferences on Autonomous Agents and Multiagent Systems}
}

@article{baroni2015automatic,
  title={Automatic evaluation of design alternatives with quantitative argumentation},
  author={Baroni, Pietro and Romano, Marco and Toni, Francesca and Aurisicchio, Marco and Bertanza, Giorgio},
  journal={Argument \& Computation},
  volume={6},
  optnumber={1},
  pages={24--49},
  year={2015},
  optpublisher={IOS Press}
}

@article{lipton1990_philosophy,
  title={Contrastive explanation},
  author={Lipton, Peter},
  journal={Royal Institute of Philosophy Supplements},
  volume={27},
  pages={247--266},
  year={1990},
  publisher={Cambridge University Press}
}

@article{hilton1990_psychology,
  title={Conversational processes and causal explanation.},
  author={Hilton, Denis J},
  journal={Psychological Bulletin},
  volume={107},
  number={1},
  pages={65},
  year={1990},
  publisher={American Psychological Association}
}

@article{van2002_social_science,
  title={Remote causes, bad explanations?},
  author={Van Bouwel, Jeroen and Weber, Erik},
  journal={Journal for the Theory of Social Behaviour},
  volume={32},
  number={4},
  year={2002}
}

@article{chin2017_cognitive,
  title={Contrastive constraints guide explanation-based category learning},
  author={Chin-Parker, Seth and Cantelon, Julie},
  journal={Cognitive science},
  volume={41},
  number={6},
  pages={1645--1655},
  year={2017},
  publisher={Wiley Online Library}
}

@inproceedings{potyka2021interpreting,
  title={Interpreting neural networks as quantitative argumentation frameworks},
  author={Potyka, Nico},
  booktitle={Proceedings of the AAAI Conference on Artificial Intelligence},
  volume={35},
  number={7},
  pages={6463--6470},
  year={2021}
}

@inproceedings{jacovi2021contrastive,
  title={Contrastive explanations for model interpretability},
  author={Jacovi, Alon and Swayamdipta, Swabha and Ravfogel, Shauli and Elazar, Yanai and Choi, Yejin and Goldberg, Yoav},
  booktitle={Proceedings of the 2021 Conference on Empirical Methods in Natural Language Processing},
  pages={1597--1611},
  year={2021}
}

@inproceedings{sundararajan2017axiomatic,
  title={Axiomatic attribution for deep networks},
  author={Sundararajan, Mukund and Taly, Ankur and Yan, Qiqi},
  booktitle={International conference on machine learning},
  pages={3319--3328},
  year={2017},
  organization={PMLR}
}

@misc{propublica2016compas,
  author       = {{ProPublica}},
  title        = {{COMPAS Recidivism Racial Bias Dataset}},
  year         = {2016},
  howpublished = {\url{https://www.kaggle.com/datasets/danofer/compass}},
  note         = {Accessed: 2025-05-06}
}

@misc{angwin2016machine,
  author       = {Angwin, Julia and Larson, Jeff and Mattu, Surya and Kirchner, Lauren},
  title        = {Machine Bias: There's Software Used Across the Country to Predict Future Criminals. And It's Biased Against Blacks},
  year         = {2016},
  howpublished = {\textit{ProPublica}},
  url          = {https://www.propublica.org/article/machine-bias-risk-assessments-in-criminal-sentencing},
  note         = {Accessed: 2025-05-06}
}

@incollection{ylikoski2007idea,
  title={The idea of contrastive explanandum},
  author={Ylikoski, Petri},
  booktitle={Rethinking explanation},
  pages={27--42},
  year={2007},
  publisher={Springer}
}

@article{miller2019_XAI,
  title={Explanation in artificial intelligence: Insights from the social sciences},
  author={Miller, Tim},
  journal={Artificial intelligence},
  volume={267},
  pages={1--38},
  year={2019},
  publisher={Elsevier}
}

@article{mossakowski2018,
  author       = {Till Mossakowski and
                  Fabian Neuhaus},
  title        = {Modular Semantics and Characteristics for Bipolar Weighted Argumentation
                  Graphs},
  journal      = {CoRR},
  volume       = {abs/1807.06685},
  year         = {2018},
  url          = {http://arxiv.org/abs/1807.06685},
  eprinttype   = {arXiv},
  eprint       = {1807.06685}
}

@inproceedings{Potyka19,
  author       = {Nico Potyka},
  editor       = {Edith Elkind and
                  Manuela Veloso and
                  Noa Agmon and
                  Matthew E. Taylor},
  title        = {Extending Modular Semantics for Bipolar Weighted Argumentation},
  booktitle    = {International Conference on Autonomous Agents
                  and MultiAgent Systems ({AAMAS})},
  pages        = {1722--1730},
  publisher    = {International Foundation for Autonomous Agents and Multiagent Systems},
  year         = {2019}
}

@inproceedings{PotykaB24,
  author       = {Nico Potyka and
                  Richard Booth},
  editor       = {Chris Reed and
                  Matthias Thimm and
                  Tjitze Rienstra},
  title        = {An Empirical Study of Quantitative Bipolar Argumentation Frameworks
                  for Truth Discovery},
  booktitle    = {Computational Models of Argument ({COMMA})},
  series       = {Frontiers in Artificial Intelligence and Applications},
  volume       = {388},
  pages        = {205--216},
  publisher    = {{IOS} Press},
  year         = {2024},
  url          = {https://doi.org/10.3233/FAIA240322},
  doi          = {10.3233/FAIA240322}}

@inproceedings{PotykaB24b,
  author       = {Nico Potyka and
                  Richard Booth},
  editor       = {Pierre Marquis and
                  Magdalena Ortiz and
                  Maurice Pagnucco},
  title        = {Balancing Open-Mindedness and Conservativeness in Quantitative Bipolar
                  Argumentation (and How to Prove Semantical from Functional Properties)},
  booktitle    = {International Conference on Principles of
                  Knowledge Representation and Reasoning ({KR})},
  year         = {2024},
  doi          = {10.24963/KR.2024/56},
}


\newpage
\setcounter{page}{1}
\onecolumn
\appendix
\section*{Supplementary Material for\\``Contrastive Explanations in Quantitative Bipolar Argumentation Frameworks''}

\section{Proofs}
\setcounter{proposition}{0}
\setcounter{property}{0}

\begin{property}[Antisymmetry]
For any $\alpha, \beta \in \mathcal{A}$ and $\gamma \in \mathcal{A} \setminus \{\alpha, \beta\}$, $\Phi_{\alpha\succeq\beta}(\gamma)=-\Phi_{\beta\succeq\alpha}(\gamma)$.
\end{property}

\begin{proposition}
If $\Phi$ satisfies antisymmetry, $\!$then $\Phi_{\alpha\succeq\alpha}(\gamma)\!\!=\!\!0$.
\end{proposition}

\begin{proof}
Antisymmetry implies $\Phi_{\alpha\succeq\alpha}(\gamma)=-\Phi_{\alpha\succeq\alpha}(\gamma)$, so $2\Phi_{\alpha\succeq\alpha}(\gamma)=0$.
Hence,
$\Phi_{\alpha\succeq\alpha}(\gamma) = 0$.
\end{proof}

\begin{property}[$\phi$-Calibration]
Let $\phi$ be an attribution function.
For any $\alpha, \beta \in \mathcal{A}$ and $\gamma \in \mathcal{A} \setminus \{\alpha, \beta\}$, 
if $\phi_{\alpha}(\gamma)= a$ and $\phi_{\beta}(\gamma) = 0$, then $\Phi_{\alpha\succeq\beta}(\gamma)= a$.
\end{property}

\begin{proposition}[Inverse $\phi$-Calibration]
\label{prop_inverse_calibration}
If $\Phi$ satisfies antisymmetry and $\phi$-Calibration, then 
$\phi_{\alpha}(\gamma)= 0$ and $\phi_{\beta}(\gamma) = b$ implies $\Phi_{\alpha\succeq\beta}(\gamma)= -b$.
\end{proposition}

\begin{proof}
We have  
$
\Phi_{\alpha\succeq\beta}(\gamma) 
= -  \Phi_{\beta \succeq \alpha}(\gamma)
= -b,
$
where we used antisymmetry for the first equality and $\phi$-Calibration with the assumption $\phi_{\beta}(\gamma) = b$ 
and $\phi_{\alpha}(\gamma)= 0$ for the second equality.
\end{proof}

\begin{property}[$\phi$-Neutrality]
Let $\phi$ be an attribution function.
For any $\alpha, \beta \in \mathcal{A}$ and $\gamma \in \mathcal{A} \setminus \{\alpha, \beta\}$
if $\phi_{\alpha}(\gamma)=\phi_{\beta}(\gamma)$, then $\Phi_{\alpha\succeq\beta}(\gamma)=0$.
\end{property}

\begin{property}[$\phi$-Monotonicity]
Let $\phi$ be an attribution function.
For any $\alpha, \beta \in \mathcal{A}$ and $\gamma \in \mathcal{A} \setminus \{\alpha, \beta\}$, if $\phi_{\alpha}(\gamma)>\phi_{\beta}(\gamma)$, then $\Phi_{\alpha\succeq\beta}(\gamma)>0$.
\end{property}

\begin{proposition}
If $\Phi$ satisfies antisymmetry and $\phi$-Monotonicity, then 
$\phi_{\alpha}(\gamma) < \phi_{\beta}(\gamma)$ implies $\Phi_{\alpha\succeq\beta}(\gamma) < 0$.    
\end{proposition}

\begin{proof}
We have $\Phi_{\alpha\succeq\beta}(\gamma) = - \Phi_{\beta \succeq \alpha}(\gamma)$.
$\phi$-Monotonicity and the assumption $\phi_{\alpha}(\gamma) < \phi_{\beta}(\gamma)$
implies $\Phi_{\beta \succeq \alpha}(\gamma) > 0$, thus $\Phi_{\alpha\succeq\beta}(\gamma) < 0$.
\end{proof}

\begin{proposition}
If $\Phi_{\alpha\succeq\beta}$ is derived from an attribution function $\phi$ using subtraction, 
then $\Phi_{\alpha\succeq\beta}$ satisfies 
Antisymmetry,
$\phi$-Calibration,
$\phi$-Neutrality and
$\phi$-Monotonicity.
\end{proposition}

\begin{proof}

Antisymmetry: If $\Phi_{\alpha\succeq\beta}$ is derived using subtraction, then
$\Phi_{\alpha\succeq\beta}(\gamma) = \phi_{\alpha}(\gamma) -  \phi_{\beta}(\gamma)
= -  (\phi_{\beta}(\gamma) - \phi_{\alpha}(\gamma)) = - \Phi_{\beta \succeq \alpha}(\gamma).$

$\phi$-Calibration: If $\phi_{\alpha}(\gamma)= a$ and $\phi_{\beta}(\gamma) = 0$, then 
$\Phi_{\alpha\succeq\beta}(\gamma)= \phi_{\alpha}(\gamma) -  \phi_{\beta}(\gamma) = a - 0 = a.$

$\phi$-Neutrality: If $\phi_{\alpha}(\gamma)=\phi_{\beta}(\gamma)$, then
$\Phi_{\alpha\succeq\beta}(\gamma) = \phi_{\alpha}(\gamma) -  \phi_{\beta}(\gamma) = 0.$  

$\phi$-Monotonicity: If $\phi_{\alpha}(\gamma) > \phi_{\beta}(\gamma)$, then
$\Phi_{\alpha\succeq\beta}(\gamma) = \phi_{\alpha}(\gamma) -  \phi_{\beta}(\gamma) > 0.$    
\end{proof}

\begin{property}[Additivity]
For any $\alpha, \beta, \eta \in \mathcal{A}$ and any $\gamma \in \mathcal{A}\setminus\{\alpha,\beta,\eta\}$, $\Phi_{\alpha\succeq\beta}(\gamma)=\Phi_{\alpha\succeq\eta}(\gamma)+\Phi_{\eta\succeq\beta}(\gamma)$.
\end{property}

\begin{proposition}
If $\Phi_{\alpha\succeq\beta}$ is derived using subtraction, then $\Phi_{\alpha\succeq\beta}$ satisfies additivity.
\end{proposition}

\begin{proof}
If $\Phi_{\alpha\succeq\beta}$ is derived using subtraction, then
$
\Phi_{\alpha\succeq\eta}(\gamma)+\Phi_{\eta\succeq\beta}(\gamma)
= \big(\phi_{\alpha}(\gamma) - \phi_{\eta}(\gamma) \big) + \big(\phi_{\eta}(\gamma) - \phi_{\beta}(\gamma) \big)
= \phi_{\alpha}(\gamma) - \phi_{\beta}(\gamma)
= \Phi_{\alpha\succeq\beta}(\gamma).
$ 
\end{proof}

\begin{proposition}
$\phi_{\alpha}^{R}, \phi_{\alpha}^{G}$ and $\phi_{\alpha}^{S}$ are plausible AFs under all modular semantics.   
\end{proposition}

\begin{proof}
As shown in \cite{PotykaB24b}, Theorem 20, all modular semantics satisfy the \emph{independence} property. 
Hence, when adding a single argument $\eta$ without connecting it to any other arguments, the strength values of existing arguments remain unchanged.
This implies immediately that the second condition holds for $\phi_{\alpha}^{R}, \phi_{\alpha}^{G}$ and $\phi_{\alpha}^{S}$ 
and that the first condition holds 
for $\phi_{\alpha}^{R}$ and $\phi_{\alpha}^{G}$ because the strength values in the differences in the definition of $\phi_{\alpha}^{R}$ and $\phi_{\alpha}^{G}$ remain unchanged.

For $\phi_{\alpha}^{S}$, satisfaction of the first condition is not obvious because the sum $\sum_{B \subseteq \Aa \setminus \{ \beta \}}
w(B) \cdot c_\beta(B)$ will now range over a large number of subsets with different weights.
Note that we can partition the subsets for the extended graph into those containing $\eta$ and those that do not contain $\eta$ and note that their number is equal
(we have
$|\{B \mid B \subseteq \Aa\setminus\{\beta\}\}|
=
|\{B \cup \{\eta\} \mid
B \subseteq \Aa\setminus\{\beta\}\}|$).
By independence, we have $c_\beta(B) = c_\beta(B \cup \{\eta\})$. Hence, the Shapley value for the extended graph is
\begin{align*}
 \sum_{B \subseteq \Aa \cup \{\eta\} \setminus \{ \beta \}} w'(B) \cdot c_\beta(B) 
&= \sum_{B \subseteq \Aa  \setminus \{ \beta \}} w'(B) \cdot c_\beta(B) 
 + \sum_{B \subseteq \Aa  \setminus \{ \beta \}} w'(B  \cup \{\eta\}) \cdot c_\beta(B  \cup \{\eta\}) \\
&= \sum_{B \subseteq \Aa  \setminus \{ \beta \}} (w'(B) + w'(B  \cup \{\eta\})) \cdot c_\beta(B),  
\end{align*}
where $w'(X) = \frac{
|X|! \cdot 
\left(|\Aa|-|X|\right)!
}{
(|\Aa| + 1)!
} $. 
We have
\begin{align*}
 w'(B) + w'(B  \cup \{\eta\}) 
 &= \frac{|B|! \cdot (|\Aa|-|B|)!}{(|\Aa| + 1)!} +  \frac{(|B|+1)! \cdot (|\Aa|-|B| -1)!}{(|\Aa| + 1)!}\\
 &= \frac{|B|! \cdot (|\Aa|-|B|-1)! \cdot((|\Aa|-|B|) + (|B|+1))}{(|\Aa| + 1)!}\\
 &= \frac{|B|! \cdot (|\Aa|-|B|-1)!}{|\Aa|!} \\
 &= w(B).
\end{align*}
Hence, $\sum_{B \subseteq \Aa \cup \{\eta\} \setminus \{ \beta \}} w'(B) \cdot c_\beta(B) = 
\sum_{B \subseteq \Aa \setminus \{ \beta \}} w(B) \cdot c_\beta(B)
= \phi_{\alpha}^{S}(\beta),
$ which completes the proof.
\end{proof}

\begin{proposition}
If $\Phi$ is a CAF derived from a plausible attribution function $\phi$, and $\Phi$ satisfies antisymmetry, $\phi$-calibration
and additivity, then, under all modular gradual semantics, $\Phi$ is equal to the 
CAF derived from $\phi$ using subtraction.
\end{proposition}
\begin{proof}
Consider an arbitrary QBAF $\mathcal{Q}$ evaluated under a modular gradual semantics, and the  QBAF $\mathcal{Q}'$ resulting from $\mathcal{Q}$
by adding an isolated argument $\eta$. 
For clarity, Condition~(2) of Plausibility is intended symmetrically: in $\mathcal{Q}'$, both $\phi_{\alpha}(\eta)=0$ and $\phi_{\eta}(\alpha)=0$ hold for every argument $\alpha$ from $\mathcal{Q}$.
First note that modularity of the gradual semantics implies that it satisfies independence \cite{PotykaB24b}. Hence, adding $\eta$ will not change the strength values of arguments in 
$\mathcal{Q}$ and by plausibility of $\phi$, the attribution values under $\phi$ will remain unchanged and $\phi_{\alpha}(\eta) = \phi_{\eta}(\alpha) = 0$ for all arguments $\alpha$ from $\mathcal{Q}$.

To distinguish CAF values under $\mathcal{Q}$ and $\mathcal{Q}'$, we write $\Phi$ and $\Phi'$, respectively.
For all arguments $\alpha, \beta, \gamma$ from $\mathcal{Q}$, 
we have
$
\Phi_{\alpha\succeq\beta}(\gamma)
= f(\phi_\alpha(\gamma), \phi_\beta(\gamma))
= \Phi'_{\alpha\succeq\beta}(\gamma)
= \Phi'_{\alpha\succeq\eta}(\gamma) + \Phi'_{\eta \succeq\beta}(\gamma)
= \phi_{\alpha}(\gamma) - \phi_{\beta}(\gamma),
$
where we used the definition of derived CAFs and plausibility for the first and second equality,
additivity for the third, and antisymmetry, $\phi$-calibration and Proposition \ref{prop_inverse_calibration} for the fourth
(since $\phi_{\eta}(\gamma) = 0$,  $\phi$-calibration implies $\Phi'_{\alpha\succeq\eta}(\gamma) = \phi_{\alpha}(\gamma)$,
and Proposition \ref{prop_inverse_calibration} implies $\Phi'_{\eta \succeq\beta}(\gamma) = - \phi_{\beta}(\gamma)$).
\end{proof}

\begin{proposition}
The CAFs 
$\Phi_{\alpha\succeq\beta}^{R}, \Phi_{\alpha\succeq\beta}^{G}, \Phi_{\alpha\succeq\beta}^{S}$
derived from the removal-based, gradient-based and Shapley-based AFs  using subtraction are defined as follows:
$$\Phi_{\alpha\succeq\beta}^{R}(\gamma)=\big(\sigma_{\mathcal{Q}}(\alpha)-\sigma_{\mathcal{Q}}(\beta)\big)- \big(\sigma_{\mathcal{Q}_{\downarrow_{\mathcal{A}'}}}(\alpha)-\sigma_{\mathcal{Q}_{\downarrow_{\mathcal{A}'}}}(\beta)\big).$$
$$\Phi_{\alpha\succeq\beta}^{G}(\gamma)= \lim_{\epsilon \to 0}\frac{\big(\sigma_{\mathcal{Q}'}(\alpha)-\sigma_{\mathcal{Q}'}(\beta)\big)- \big(\sigma_{\mathcal{Q}}(\alpha)-\sigma_{\mathcal{Q}}(\beta)\big)}{\epsilon}.$$
$$
\Phi_{\alpha\succeq\beta}^{S}(\gamma)= \sum_{B \subseteq \A \setminus \{ \alpha, \beta, \gamma \}} \big(
w(B) \cdot ( c_{\gamma \rightarrow \alpha}(B) - c_{\gamma \rightarrow \beta}(B))
+
w(B \cup \{\beta\}) \cdot ( c_{\gamma \rightarrow \alpha}(B \cup \{\beta\}) - c_{\gamma \rightarrow \beta}(B \cup \{\alpha\}))
\big), 
$$
where $c_{\gamma \rightarrow x}(X) = 
\sigma_{\mathcal{Q}_{\downarrow_{X \cup \{x, \gamma\}}}}(x)
-
\sigma_{\mathcal{Q}_{\downarrow_{X \cup \{ x \}}}}(x)
$.
\end{proposition}
\begin{proof}
1. 
 $\Phi_{\alpha\succeq\beta}^{R}(\gamma)
 =\big(\sigma_{\mathcal{Q}}(\alpha)-\sigma_{\mathcal{Q}}(\beta)\big)- \big(\sigma_{\mathcal{Q}_{\downarrow_{\mathcal{A}'}}}(\alpha)-\sigma_{\mathcal{Q}_{\downarrow_{\mathcal{A}'}}}(\beta)\big)
 =  \big(\sigma_{\mathcal{Q}}(\alpha)-\sigma_{\mathcal{Q}_{\downarrow_{\mathcal{A}'}}}(\alpha)\big)- \big(\sigma_{\mathcal{Q}}(\beta)-\sigma_{\mathcal{Q}_{\downarrow_{\mathcal{A}'}}}(\beta)\big)
 = \phi_{\alpha}^{R}(\gamma) - \phi_{\beta}^{R}(\gamma).
 $

2. Here, $\mathcal{Q}'$ denotes the perturbed QBAF $\mathcal{Q}_{\epsilon}$ defined in Definition~\ref{gradient_aae}.

$\Phi_{\alpha\succeq\beta}^{G}(\gamma)
= \lim_{\epsilon \to 0}\frac{\big(\sigma_{\mathcal{Q}'}(\alpha)-\sigma_{\mathcal{Q}'}(\beta)\big)- \big(\sigma_{\mathcal{Q}}(\alpha)-\sigma_{\mathcal{Q}}(\beta)\big)}{\epsilon}
= \lim_{\epsilon \to 0}\frac{\sigma_{\mathcal{Q}'}(\alpha) -\sigma_{\mathcal{Q}}(\alpha)}{\epsilon}
- \lim_{\epsilon \to 0}\frac{\sigma_{\mathcal{Q}'}(\beta) - \sigma_{\mathcal{Q}}(\beta)}{\epsilon}
= \phi_{\alpha}^{G}(\gamma) - \phi_{\beta}^{G}(\gamma),
$
where the second equality follows from linearity of limits.

3.  We have
\begin{align*}
  \Phi_{\alpha\succeq\beta}^{S}(\gamma)
  &=  \phi_{\alpha}^{S}(\gamma) - \phi_{\beta}^{S}(\gamma) \\
&=
\sum_{B \subseteq \A \setminus \{ \alpha, \gamma \}} w(B) \cdot c_{\gamma \rightarrow \alpha}(B)
- \sum_{B \subseteq \A \setminus \{ \beta, \gamma \}} w(B) \cdot c_{\gamma \rightarrow \beta}(B) \\
&=
\sum_{B \subseteq \A \setminus \{ \alpha, \beta, \gamma \}} \big( w(B) \cdot c_{\gamma \rightarrow \alpha}(B) + w(B \cup \{\beta\}) \cdot c_{\gamma \rightarrow \alpha}(B \cup \{\beta\}) \big) \\
&\ - \sum_{B \subseteq \A \setminus \{ \alpha, \beta, \gamma \}} \big( w(B) \cdot c_{\gamma \rightarrow \beta}(B) + w(B \cup \{\alpha\}) \cdot c_{\gamma \rightarrow \beta}(B \cup \{\alpha\}) \big) \\
&=
\sum_{B \subseteq \A \setminus \{ \alpha, \beta, \gamma \}} \big(
w(B) \cdot ( c_{\gamma \rightarrow \alpha}(B) - c_{\gamma \rightarrow \beta}(B))
+
w(B \cup \{\beta\}) \cdot ( c_{\gamma \rightarrow \alpha}(B \cup \{\beta\}) - c_{\gamma \rightarrow \beta}(B \cup \{\alpha\}))
\big),
\end{align*}
where, for the last equality, we used the fact that the value of $w(X)$ depends only on the size of $X$ and therefore
$w(B \cup \{\beta\}) = w(B \cup \{\alpha\})$.
\end{proof}
\begin{proposition}
If $\Phi_{\alpha \succeq \beta}$ satisfies additivity,
Algorithm \ref{algo_caf} computes $\Phi_{t_i \succeq t_j}(\gamma)$ for all $1 \leq i < j \leq T$
with $O(T)$  $\Phi_{\alpha \succeq \beta}$-calls.
If $\Phi_{\alpha \succeq \beta}$ can be computed in time $O(C)$,
the overall time complexity of the algorithm
is $O(T\cdot C + T^2)$.
\end{proposition}
\begin{proof}
For the number of CAF calls, note that the algorithm uses $T-1 = O(T)$ $\Phi_{\alpha \succeq \beta}$-calls in the for-loop from line 2 to 4 and at no
other place. 

To see that all values are computed correctly, think of the values as organised in a $T \times T$ matrix.
In lines 2 to 4, we initialise the super diagonal. 
We have to assure that the matrix values are correctly initialised in lines 5 to 9.
We let $M[i,j] = M[i,j-1] + M[j-1, j]$. Hence, if 
$M[i,j-1] = \Phi_{t_i \succeq t_{j-1}}(\gamma)$ and
$M[j-1,j] = \Phi_{t_{j-1}\succeq t_{j}}(\gamma)$,
Additivity implies that $M[i,j] = \Phi_{t_{i}\succeq t_{j}}(\gamma)$
as desired. $M[j-1,j] = \Phi_{t_{j-1}\succeq t_{j}}(\gamma)$ follows immediately from lines 2 to 4,
hence it remains to check that $M[i,j-1] = \Phi_{t_i \succeq t_{j-1}}(\gamma)$.
For every $1\leq i\leq T-2$, we prove by induction on $j$ that
$M[i,j]=\Phi_{t_i\succeq t_j}(\gamma)$.
For the induction base $j=i+2$, the entries $M[i,i+1]$ and
$M[i+1,i+2]$ are correctly initialised in lines 2 to 4.
Hence, by Additivity,
$M[i,i+2]=M[i,i+1]+M[i+1,i+2]
=\Phi_{t_i\succeq t_{i+2}}(\gamma)$.
For the induction step, assume that the claim holds for all
$i+2\leq j\leq N$, where $N<T$.
Then, for $j=N+1$, we compute
$M[i,N+1]=M[i,N]+M[N,N+1]$.
We already established that $M[N,N+1]$ is correctly initialised
in lines 2 to 4, and the induction assumption implies that
$M[i,N]$ is computed correctly. Hence, Additivity implies that
$M[i,N+1]$ is computed correctly, which completes the correctness proof.

For the time complexity, lines 2 to 4 run in time $O(T\cdot C)$.
In lines 5 to 9, we fill $T-2$ entries for the first row, $T-3$ for the second and so on.
Hence, overall, we have to fill $\sum_{k=1}^{T-2} k = \frac{(T-2)\cdot(T-1)}{2}$ entries.
Each entry requires one addition, hence the overall number of operations is $O(T^2)$.
\end{proof}

\begin{property}[Counterfactuality]
$\Phi_{\alpha\succeq\beta}(\gamma)$ satisfies \emph{Counterfactuality}
iff, for any $\gamma\in\mathcal{A}\setminus\{\alpha, \beta\}$, letting $\mathcal{A}'=\mathcal{A} \setminus \{\gamma\}$,
the following statements hold:\\
1. If $\Phi_{\alpha\succeq\beta}(\gamma)<0$, then $\sigma_{\mathcal{Q}}(\alpha)-\sigma_{\mathcal{Q}}(\beta)<\sigma_{\mathcal{Q}_{\downarrow_{\mathcal{A}'}}}(\alpha)-\sigma_{\mathcal{Q}_{\downarrow_{\mathcal{A}'}}}(\beta)$;\\
2. If $\Phi_{\alpha\succeq\beta}(\gamma)>0$, then $\sigma_{\mathcal{Q}}(\alpha)-\sigma_{\mathcal{Q}}(\beta)>\sigma_{\mathcal{Q}_{\downarrow_{\mathcal{A}'}}}(\alpha)-\sigma_{\mathcal{Q}_{\downarrow_{\mathcal{A}'}}}(\beta)$.
\end{property}

\begin{proposition}
$\Phi_{\alpha\succeq\beta}^{R}(\gamma)$ satisfies Counterfactuality, while $\Phi_{\alpha\succeq\beta}^{G}(\gamma)$ and $\Phi_{\alpha\succeq\beta}^{S}(\gamma)$ can violate Counterfactuality.
\end{proposition}

\begin{proof}
If $\Phi^{R}_{\alpha\succeq\beta}(\gamma)<0$, then $\big(\sigma_{\mathcal{Q}}(\alpha)-\sigma_{\mathcal{Q}}(\beta)\big)- \big(\sigma_{\mathcal{Q}_{\downarrow_{\mathcal{A}'}}}(\alpha)-\sigma_{\mathcal{Q}_{\downarrow_{\mathcal{A}'}}}(\beta)\big)<0$, that is, $\big(\sigma_{\mathcal{Q}}(\alpha)-\sigma_{\mathcal{Q}}(\beta)\big)< \big(\sigma_{\mathcal{Q}_{\downarrow_{\mathcal{A}'}}}(\alpha)-\sigma_{\mathcal{Q}_{\downarrow_{\mathcal{A}'}}}(\beta)\big)$.
The case where $\Phi_{\alpha\succeq\beta}^{R}(\gamma)>0$ follows analogously. \\

To show that the gradient- and Shapley-based CAFs may violate Counterfactuality, consider the QBAF \(\mathcal Q=\langle\mathcal A,\mathcal R^-,\mathcal R^+,\tau\rangle\), where \(\mathcal A=\{\alpha,\beta,\gamma,\eta\}\), \(\mathcal R^+=\{(\eta,\alpha)\}\), \(\mathcal R^-=\{(\gamma,\eta),(\gamma,\beta),(\eta,\beta),(\beta,\alpha)\}\), and \(\tau(x)=1/2\) for every \(x\in\mathcal A\). Under DF-QuAD semantics, we obtain \(\sigma_{\mathcal Q}(\alpha)=17/32\) and \(\sigma_{\mathcal Q}(\beta)=3/16\). Therefore, \(\sigma_{\mathcal Q}(\alpha)-\sigma_{\mathcal Q}(\beta)=11/32\).

Let \(\mathcal A'=\mathcal A\setminus\{\gamma\}\). After removing \(\gamma\), we obtain \(\sigma_{\mathcal Q_{\downarrow_{\mathcal A'}}}(\alpha)=5/8\) and \(\sigma_{\mathcal Q_{\downarrow_{\mathcal A'}}}(\beta)=1/4\). Hence, \(\sigma_{\mathcal Q_{\downarrow_{\mathcal A'}}}(\alpha)-\sigma_{\mathcal Q_{\downarrow_{\mathcal A'}}}(\beta)=3/8\).

A direct calculation gives \(\Phi_{\alpha\succeq\beta}^{G}(\gamma)=1/8>0\). For the Shapley-based CAF, we obtain \(\phi_{\alpha}^{S}(\gamma)=-1/32\) and \(\phi_{\beta}^{S}(\gamma)=-5/32\). Therefore, \(\Phi_{\alpha\succeq\beta}^{S}(\gamma)=\phi_{\alpha}^{S}(\gamma)-\phi_{\beta}^{S}(\gamma)=1/8>0\). However, \(11/32<3/8\). Thus, both attribution scores are positive even though removing \(\gamma\) increases, rather than decreases, the strength difference between \(\alpha\) and \(\beta\). Therefore, \(\Phi_{\alpha\succeq\beta}^{G}\) and \(\Phi_{\alpha\succeq\beta}^{S}\) may violate Counterfactuality.
\end{proof}

\begin{property}[Local Faithfulness]
\label{property_faithfulness}
$\Phi_{\alpha\succeq\beta}(\gamma)$ satisfies \emph{Local Faithfulness} wrt. $\sigma$ iff, for any $\gamma \in \mathcal{A}\setminus\{\alpha,\beta\}$,
there exists $\delta>0$ such that, for all $e \in [\tau(\gamma)-\delta, \tau(\gamma)+\delta] \cap [0,1]$, 
letting $\mathcal{Q}'=\left\langle\mathcal{A}, \mathcal{R}^{-}, \mathcal{R}^{+}, \tau' \right\rangle$ be the QBAF such that $\tau'(\gamma)=e$ and $\tau'(\eta)=\tau(\eta)$ for all $\eta \in \mathcal{A}\setminus\{\gamma\}$.
the following statements hold:\\
1. If $\Phi_{\alpha\succeq\beta}(\gamma)<0$, then $\sigma_{\mathcal{Q}}(\alpha)-\sigma_{\mathcal{Q}}(\beta)\leq\sigma_{\mathcal{Q}'}(\alpha)-\sigma_{\mathcal{Q}'}(\beta)$ whenever $e<\tau(\gamma)$, and $\sigma_{\mathcal{Q}}(\alpha)-\sigma_{\mathcal{Q}}(\beta)\geq\sigma_{\mathcal{Q}'}(\alpha)-\sigma_{\mathcal{Q}'}(\beta)$ whenever $e>\tau(\gamma)$;\\
2. If $\Phi_{\alpha\succeq\beta}(\gamma)>0$, then $\sigma_{\mathcal{Q}}(\alpha)-\sigma_{\mathcal{Q}}(\beta)\geq\sigma_{\mathcal{Q}'}(\alpha)-\sigma_{\mathcal{Q}'}(\beta)$ whenever $e<\tau(\gamma)$, and $\sigma_{\mathcal{Q}}(\alpha)-\sigma_{\mathcal{Q}}(\beta)\leq\sigma_{\mathcal{Q}'}(\alpha)-\sigma_{\mathcal{Q}'}(\beta)$ whenever $e>\tau(\gamma)$.
\end{property}

\begin{proposition}
$\Phi_{\alpha\succeq\beta}^{G}(\gamma)$ satisfies Local Faithfulness, while $\Phi_{\alpha\succeq\beta}^{R}(\gamma)$ and $\Phi_{\alpha\succeq\beta}^{S}(\gamma)$ can violate Local Faithfulness.
\end{proposition}

\begin{proof}
If $\Phi_{\alpha\succeq\beta}^{G}(\gamma)=\lim_{\epsilon \to 0}\frac{\big(\sigma_{\mathcal{Q}'}(\alpha)-\sigma_{\mathcal{Q}'}(\beta)\big)- \big(\sigma_{\mathcal{Q}}(\alpha)-\sigma_{\mathcal{Q}}(\beta)\big)}{\epsilon}<0$, then there exists $\delta>0$ such that for any $0<|\epsilon|<\delta$, we have $\frac{\big(\sigma_{\mathcal{Q}'}(\alpha)-\sigma_{\mathcal{Q}'}(\beta)\big)- \big(\sigma_{\mathcal{Q}}(\alpha)-\sigma_{\mathcal{Q}}(\beta)\big)}{\epsilon}=\frac{\big(\sigma_{\mathcal{Q}'}(\alpha)-\sigma_{\mathcal{Q}'}(\beta)\big)- \big(\sigma_{\mathcal{Q}}(\alpha)-\sigma_{\mathcal{Q}}(\beta)\big)}{e-\tau(\gamma)}<0$. If $e<\tau(\gamma)$, then $\sigma_{\mathcal{Q}}(\alpha)-\sigma_{\mathcal{Q}}(\beta)\leq\sigma_{\mathcal{Q}'}(\alpha)-\sigma_{\mathcal{Q}'}(\beta)$. The remaining three cases follow analogously.\\

To show that the removal- and Shapley-based CAFs may violate Local Faithfulness, consider the QBAF \(\mathcal Q=\langle\mathcal A,\mathcal R^-,\mathcal R^+,\tau\rangle\), where \(\mathcal A=\{\alpha,\beta,\gamma,\eta\}\), \(\mathcal R^+=\{(\gamma,\eta)\}\), \(\mathcal R^-=\{(\gamma,\alpha),(\gamma,\beta),(\eta,\beta),(\beta,\alpha)\}\), and \(\tau(x)=1/2\) for every \(x\in\mathcal A\). Under DF-QuAD semantics, we obtain \(\sigma_{\mathcal Q}(\alpha)=15/64\) and \(\sigma_{\mathcal Q}(\beta)=1/16\), and hence \(\sigma_{\mathcal Q}(\alpha)-\sigma_{\mathcal Q}(\beta)=11/64\).

Let \(\mathcal A'=\mathcal A\setminus\{\gamma\}\). After removing \(\gamma\), we obtain \(\sigma_{\mathcal Q_{\downarrow_{\mathcal A'}}}(\alpha)=3/8\) and \(\sigma_{\mathcal Q_{\downarrow_{\mathcal A'}}}(\beta)=1/4\). Therefore, \(\Phi_{\alpha\succeq\beta}^{R}(\gamma)=11/64-(3/8-1/4)=3/64>0\).

However, a direct calculation gives \(\Phi_{\alpha\succeq\beta}^{G}(\gamma)=-5/32<0\). Hence, for every sufficiently small increase \(e>\tau(\gamma)\), the strength difference between \(\alpha\) and \(\beta\) decreases. This contradicts Local Faithfulness for the positive removal-based attribution \(\Phi_{\alpha\succeq\beta}^{R}(\gamma)>0\).

For the Shapley-based CAF, we obtain \(\phi_{\alpha}^{S}(\gamma)=-35/192\) and \(\phi_{\beta}^{S}(\gamma)=-7/32\). Hence, \(\Phi_{\alpha\succeq\beta}^{S}(\gamma)=\phi_{\alpha}^{S}(\gamma)-\phi_{\beta}^{S}(\gamma)=7/192>0\). Since the strength difference decreases for arbitrarily small increases in the base score of \(\gamma\), this positive attribution also violates Local Faithfulness. Therefore, \(\Phi_{\alpha\succeq\beta}^{R}\) and \(\Phi_{\alpha\succeq\beta}^{S}\) may violate Local Faithfulness.

\end{proof}


\begin{property}[Cross-topic-adjusted Efficiency]
\label{property_efficiency}
$\Phi_{\alpha\succeq\beta}(\gamma)$ satisfies \emph{Cross-topic-adjusted Efficiency} iff \(\sum_{\gamma\in\mathcal A\setminus\{\alpha,\beta\}}\Phi_{\alpha\succeq\beta}(\gamma)+\bigl(\phi_{\alpha}(\beta)-\phi_{\beta}(\alpha)\bigr)=\bigl(\sigma_{\mathcal Q}(\alpha)-\sigma_{\mathcal Q}(\beta)\bigr)-\bigl(\tau(\alpha)-\tau(\beta)\bigr)\).
\end{property}

\begin{proposition}
\label{prop_method_efficiency}
$\Phi_{\alpha\succeq\beta}^{S}(\gamma)$ satisfies Cross-topic-adjusted Efficiency, while $\Phi_{\alpha\succeq\beta}^{R}(\gamma)$ and $\Phi_{\alpha\succeq\beta}^{G}(\gamma)$ can violate it.
\end{proposition}

\begin{proof}
Since the Shapley-based CAF is derived using subtraction, we have $\Phi_{\alpha\succeq\beta}^{S}(\gamma)=\phi_{\alpha}^{S}(\gamma)-\phi_{\beta}^{S}(\gamma)$ for every $\gamma\in\mathcal A\setminus\{\alpha,\beta\}$. Therefore, $\sum_{\gamma\in\mathcal A\setminus\{\alpha,\beta\}}\Phi_{\alpha\succeq\beta}^{S}(\gamma)+\bigl(\phi_{\alpha}^{S}(\beta)-\phi_{\beta}^{S}(\alpha)\bigr)=\sum_{\gamma\in\mathcal A\setminus\{\alpha\}}\phi_{\alpha}^{S}(\gamma)-\sum_{\gamma\in\mathcal A\setminus\{\beta\}}\phi_{\beta}^{S}(\gamma)$. By the Efficiency property of individual Shapley attributions, established in~\cite{kampik2024contribution}, this is equal to $\bigl(\sigma_{\mathcal Q}(\alpha)-\tau(\alpha)\bigr)-\bigl(\sigma_{\mathcal Q}(\beta)-\tau(\beta)\bigr)$, which is equivalent to $\bigl(\sigma_{\mathcal Q}(\alpha)-\sigma_{\mathcal Q}(\beta)\bigr)-\bigl(\tau(\alpha)-\tau(\beta)\bigr)$. Hence, $\Phi_{\alpha\succeq\beta}^{S}$ satisfies Cross-topic-adjusted Efficiency.

To show that the removal- and gradient-based CAFs may violate Cross-topic-adjusted Efficiency, consider the QBAF $\mathcal Q=\langle\mathcal A,\mathcal R^-,\mathcal R^+,\tau\rangle$, where $\mathcal A=\{\alpha,\beta,\gamma,\eta\}$, $\mathcal R^-=\emptyset$, $\mathcal R^+=\{(\gamma,\alpha),(\eta,\alpha)\}$, and $\tau(x)=1/2$ for every $x\in\mathcal A$. Under DF-QuAD semantics, $\sigma_{\mathcal Q}(\alpha)=7/8$ and $\sigma_{\mathcal Q}(\beta)=1/2$. Hence, the difference between the final-strength gap and the base-score gap is $3/8$.

For the removal-based CAF, $\phi_{\alpha}^{R}(\gamma)=\phi_{\alpha}^{R}(\eta)=1/8$, while all corresponding attributions to $\beta$ and both cross-topic attributions are $0$. Therefore, the left-hand side of Cross-topic-adjusted Efficiency is $1/8+1/8=1/4$, which is different from $3/8$.

For the gradient-based CAF, $\phi_{\alpha}^{G}(\gamma)=\phi_{\alpha}^{G}(\eta)=1/4$, while all corresponding attributions to $\beta$ and both cross-topic attributions are $0$. Therefore, the left-hand side of Cross-topic-adjusted Efficiency is $1/4+1/4=1/2$, which is also different from $3/8$. Hence, $\Phi_{\alpha\succeq\beta}^{R}$ and $\Phi_{\alpha\succeq\beta}^{G}$ may violate Cross-topic-adjusted Efficiency.
\end{proof}

\section{Additional Contrastive Explanations in Section \ref{sec_case_study}}
\label{sec_additional_attributions_appendix}

\begin{figure*}[h]
    \centering
    \includegraphics[width=1.0\linewidth]{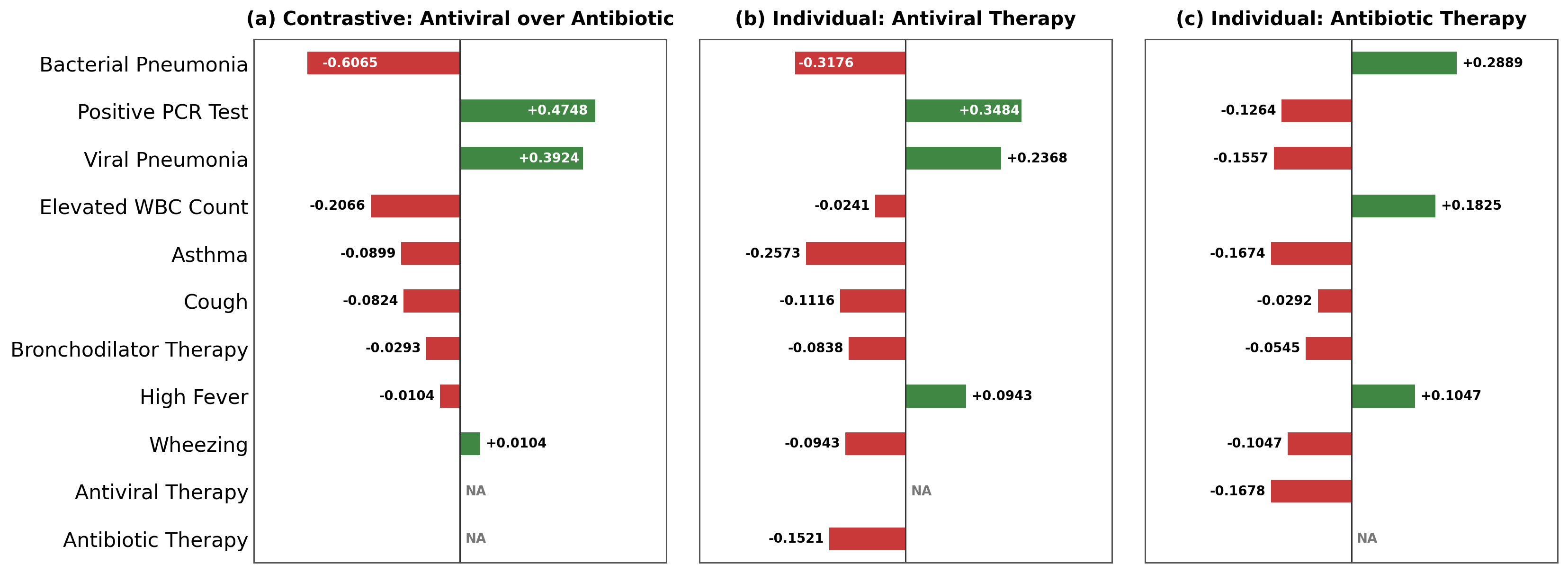}
    \caption{Gradient-based contrastive and individual explanations for treatment selection. Green bars indicate positive influence, while red bars indicate negative influence.}
    \label{fig_case_study_1113}
\end{figure*}

\section{Computing Environment}
All experiments were conducted locally on a personal laptop running Microsoft Windows 11 Home, equipped with an Intel Core Ultra 5 225H CPU (14 cores) and 32 GB of RAM. All computations were performed on the CPU. The software environment consisted of Python 3.11.9, PyTorch 2.13.0, NumPy 2.4.4, pandas 3.0.3, scikit-learn 1.9.0, and Matplotlib 3.10.8.

\end{document}